\documentclass{article}
 
\usepackage{microtype}
\usepackage{graphicx}
\usepackage{subcaption}
\usepackage{booktabs} % for professional tables
\usepackage{pifont}
\usepackage{hyperref}
\usepackage{algorithm}
\usepackage{algorithmic}
 
\newcommand{\E}{\mathbb{E}}
\newcommand{\Var}{\mathrm{Var}}
\newcommand{\Tr}{\mathrm{Tr}}
\newcommand{\bSigma}{\boldsymbol{\Sigma}}
\newcommand{\hatSigma}{\hat{\boldsymbol{\Sigma}}}
\newcommand{\TailPenalty}[1]{\left(1 + \frac{3}{#1 - 4}\right)}
\usepackage[preprint]{icml2026}
 
\usepackage{amsmath}
\usepackage{amssymb}
\usepackage{mathtools}
\usepackage{amsthm}

\usepackage[capitalize,noabbrev]{cleveref}
 
\theoremstyle{plain}
\newtheorem{theorem}{Theorem}[section]
\newtheorem{proposition}[theorem]{Proposition}
\newtheorem{lemma}[theorem]{Lemma}

\theoremstyle{definition}
\newtheorem{definition}[theorem]{Definition}
\newtheorem{assumption}[theorem]{Assumption}
\theoremstyle{remark}

\usepackage[textsize=tiny]{todonotes}
 
\icmltitlerunning{Sample Complexity of Causal Identification with Temporal Heterogeneity
}
 
\begin{document}
 
\twocolumn[
  \icmltitle{Sample Complexity of Causal Identification with Temporal Heterogeneity}

  \begin{center}
    % First row of authors
    \begin{tabular}{c p{0.4in} c p{0.4in} c}
      {\bf Ameya Rathod} & & {\bf Sujay Belsare}$^*$ & & {\bf Salvik Krishna Nautiyal}$^*$ \\
      \small IIIT Hyderabad & & \small IIIT Hyderabad & & \small IIIT Hyderabad
    \end{tabular}

    \vspace{0.2in} % Spacing between the two rows of authors

    % Second row of authors centered
    \begin{tabular}{c p{0.4in} c}
      {\bf Dhruv Laad} & & {\bf Ponnurangam Kumaraguru} \\
      \small AlphaGrep & & \small IIIT Hyderabad
    \end{tabular}
  \end{center}
\icmlcorrespondingauthor{Ameya Rathod}{ameya.rathod@research.iiit.ac.in}
  \icmlcorrespondingauthor{Sujay Belsare}{sujay.belsare@students.iiit.ac.in}
  % \icmlkeywords{Machine Learning, ICML}
  \vskip 0.3in
]

% Place correspondence and equal contribution in the footnote
\printAffiliationsAndNotice{\icmlEqualContribution}
 
\begin{abstract}
  % This document provides a basic paper template and submission guidelines.
  % Abstracts must be a single paragraph, ideally between 4--6 sentences long.
  % Gross violations will trigger corrections at the camera-ready phase.
    Recovering a unique causal graph from observational data is an ill-posed problem because  multiple generating mechanisms can lead to the same observational distribution. This problem becomes solvable only by exploiting specific structural or distributional assumptions. While recent work has separately utilized time-series dynamics or multi-environment heterogeneity to constrain this problem, we integrate both as complementary sources of heterogeneity. This integration yields unified necessary identifiability conditions and enables a rigorous analysis of the statistical limits of recovery under thin versus heavy-tailed noise. In particular, temporal structure is shown to effectively substitute for missing environmental diversity, possibly achieving identifiability even under \textit{insufficient heterogeneity}. Extending this analysis to heavy-tailed (Student's t) distributions, we demonstrate that while geometric identifiability conditions remain invariant, the sample complexity diverges significantly from the Gaussian baseline. Explicit information-theoretic bounds quantify this cost of robustness, establishing the fundamental limits of covariance-based causal graph recovery methods in realistic non-stationary systems. 
    This work shifts the focus from whether causal structure is identifiable to whether it is statistically recoverable in practice.\footnote{The code for experimentation is available \href{https://anonymous.4open.science/r/causal-iden-sample-complexity-6CEC}{here}.}
\end{abstract}
 
\section{Introduction}
\label{sec:introduction}
 
Causal structure learning aims to reconstruct the underlying data-generating mechanism of a system solely from observational evidence, providing the fundamental framework necessary to answer counterfactual questions and predict the outcome of unseen interventions~\cite{pmlr-v6-pearl10a}. While the causal structure is typically formalized as a Directed Acyclic Graph (DAG), recovering this graph from finite data poses a formidable challenge: multiple distinct mechanisms can yield identical observational distributions, restricting recovery to a Markov Equivalence Class (MEC) rather than a unique graph. To resolve this ambiguity, extensive research has focused on specific structural and distributional assumptions, identifying that unconstrained recovery is theoretically impossible \cite{10.5555/3202377}.
 
While existing literature on multi-environmental identifiability simplifies analysis considering Gaussian noise, this simplification ignores the reality of various systems like financial markets, where data is strictly heavy-tailed (e.g., Student's t distributed) and extreme events are frequent. While it is well-established that non-Gaussianity (e.g., LiNGAM) aids identification via higher-order moments~\cite{10.5555/1248547.1248619}, the specific subset of heavy-tailed distributions presents a double-edged sword: they provide identifiability signals via non-linear score functions but impose severe stability challenges for estimation due to infinite or high variance~\cite{catoni2011challengingempiricalmeanempirical}. Consequently, the interaction between multi-environment identifiability and heavy-tailed noise remains significantly underexplored.
 
Our contributions are:
\begin{itemize}
    \item We show that insufficient environmental heterogeneity, available as rank-deficient regime shifts, can be fully compensated by temporal heteroskedasticity, extending \cite{montagna2025identifiabilitycausalgraphsmultiple}. We further establish a fundamental lower bound on the computational window required to recover full system rank, even when each time step provides maximal novel information. 
    \item We extend identifiability guarantees to multivariate Student's t noise, showing that the geometric conditions for recovery are invariant to tail behavior. We further provide finite-sample guarantees and a local minimax lower bound demonstrating that heavy-tailed noise induces an unavoidable sample-complexity penalty for all covariance-based estimators. 
\end{itemize}
Our analysis combines temporal and multi-environmental heterogeneity to study both identifiability and statistical efficiency. While these sources of heterogeneity suffice to identify the causal structure, we show that the resulting sample complexity depends critically on tail behavior, with heavy-tailed noise imposing an information-theoretic barrier for approaches that rely on higher-order moments. To the best of our knowledge, this is the first work to provide explicit sample-complexity analysis for causal identification using second-order moments under heavy-tailed noise.
\section{Related Works}
\paragraph{Causal Identification in Time Series.}
The time-series paradigm distinguishes between lagged dependencies (past events influencing the future) and instantaneous dependencies (concurrent interactions)~\cite{shojaie2021grangercausalityreviewrecent,hasan2024surveycausaldiscoverymethods,pmlr-v216-gunther23a}. While effective, these methods often rely on purely temporal precedence or strict stationarity, failing to exploit structural information available from environmental changes. Classical approaches, such as Granger Causality \cite{shojaie2021grangercausalityreviewrecent}, focus on predictive ability but lack a structural causal interpretation. Modern frameworks typically model the system using Structural Causal Models (SCMs), utilizing constraint-based methods like PCMCI \cite{Runge_2019} or functional causal models like VarLiNGAM \cite{JMLR:v11:hyvarinen10a} to recover directionality. 
\paragraph{Causal Identification in multi-environmental setting.}Multi-environment approaches leverage distributional shifts across distinct regimes (e.g., markets or experimental settings) to extract sufficient structural information for unique identifiability, utilizing the principle of invariant causal mechanisms~\cite{montagna2025identifiabilitycausalgraphsmultiple,pmlr-v286-kivva25a,peters2015causalinferenceusinginvariant}. However, a significant gap remains in the distributional assumptions governing these models. The vast majority of multi-environment literature operates under the Gaussian assumption for mathematical tractability, as variance shifts in linear Gaussian systems theoretically break the symmetry of the MEC~\cite{10.5555/3327345.3327524}.
 
 We provide a comparison of related methods for causal identification in Table \ref{compare}.
\begin{table*}[t]
\centering
\caption{Comparison of Causal Discovery Frameworks. We categorize methods based on their ability to exploit Temporal dynamics or Environmental heterogeneity, their identifiability of instantaneous effects, and their distributional assumptions. ($\checkmark$ = Fully Supported, $\circ$ = Partial/Restricted, \ding{55} = Not Supported).}
\label{tab:related_works_comparison}
\resizebox{\textwidth}{!}{%
\begin{tabular}{@{}lccccc@{}}
\toprule
\textbf{Method} & \textbf{Temporal} & \textbf{Multi-Env} & \textbf{Instantaneous} & \textbf{Distributional} & \textbf{Primary Source of} \\
& \textbf{Dynamics} & \textbf{Heterogeneity} & \textbf{Identifiability} & \textbf{Assumption} & \textbf{Identifiability} \\ \midrule
{\textbf{\textit{Time-Series Approaches}}} \\ 
\midrule
Granger Causality~\cite{shojaie2021grangercausalityreviewrecent} & $\checkmark$ & \ding{55} & \ding{55} & Linear & Predictive Lag \\
PCMCI~\cite{Runge_2019} & $\checkmark$ & \ding{55} & $\circ$  & Any (Indep. Test) & CI Tests + Time Order \\
VarLiNGAM~\cite{JMLR:v11:hyvarinen10a} & $\checkmark$ & \ding{55} & $\checkmark$ & Non-Gaussian & Higher-Order Moments \\ \midrule
{\textbf{\textit{Multi-Environment Approaches}}} \\ 
\midrule
ICP / Invariant Pred.~\cite{peters2015causalinferenceusinginvariant} & \ding{55} & $\checkmark$ & $\circ$ & Any & Invariance Principle \\
Linear Multi-Env~\cite{10.5555/3327345.3327524} & \ding{55} & $\checkmark$ & $\checkmark$ & Gaussian & Variance Shifts \\
Montagna et al.~\cite{montagna2025identifiabilitycausalgraphsmultiple} & \ding{55} & $\checkmark$ & $\checkmark$ & Gaussian & Rank-Sufficient Shifts \\ \midrule
\textbf{Our Analysis} & $\checkmark$ & $\checkmark$ & $\checkmark$ & \textbf{Heavy-Tailed (t-distribution)} & \textbf{Environmental and Temporal Heterogeneity} \\ \bottomrule
\end{tabular}%
}
\label{compare}
\end{table*}
\section{Preliminaries}
\label{sec:preliminaries}
 
In this section, we introduce the necessary background on Structural Equation Models (SEMs) \cite{pmlr-v6-pearl10a}, define the assumptions governing non-stationarity, and specify the distributional properties utilized in our identifiability theorems.
 
\subsection{Causal Graphical Models}
We represent the causal structure of a system of $d$ observable variables $X_t = [X_{1,t}, \dots, X_{d,t}]^\top \in \mathbb{R}^d$ as a \textbf{Directed Acyclic Graph (DAG)} $\mathcal{G} = (\mathcal{V}, \mathcal{E})$.
\begin{itemize}
    \item The vertex set $\mathcal{V} = \{1, \dots, d\}$ corresponds to the observed variables.
    \item The edge set $\mathcal{E}$ contains directed edges, where $(i, j) \in \mathcal{E}$ denotes a direct causal effect from $X_{i,t}$ to $X_{j,t}$.
\end{itemize}
 
The causal relationships are parameterized by a weighted adjacency matrix $B \in \mathbb{R}^{d \times d}$, such that $B_{ij} \neq 0$ if and only if $(i, j) \in \mathcal{E}$. To ensure the graph is acyclic, we assume there exists a permutation $\pi$ of the variables such that $B$ can be permuted to a strictly lower triangular matrix.
 
\subsection{Structural Equation Model (SEM)}
We assume the data generating process follows a linear SEM with instantaneous effects. At any time step $t$, the value of $X_{i,t}$ is determined by a linear combination of its causal parents and an independent noise term:
\begin{equation}
    X_{i,t} = \sum_{j \in \text{Pa}(i)} B_{ji} X_{j,t} + \epsilon_{i,t}
\end{equation}
where Pa($i$) refer to the parents of $i$ in the causal graph.
In matrix form, the system is expressed as:
\begin{equation}
    X_t = B^\top X_t + \epsilon_t
\end{equation}
where $\epsilon_t = [\epsilon_{1,t}, \dots, \epsilon_{d,t}]^\top$ is the vector of mutually independent structural shocks (noise).
 
By defining the \textbf{mixing matrix} $A = (I - B^\top)^{-1}$, we can express the observed variables as a linear transformation of the independent noise terms:
\begin{equation}
    X_t = A \epsilon_t
\end{equation}
Our primary objective is to recover the adjacency matrix $B$ (or equivalently $A$) from the observed data sequence $\{X_t\}$.
\begin{definition}
Given a base ICA (Independent Component Analysis) model $\mathbf{X} = f(\mathbf{S})$ with $\mathbf{S} \sim p_\omega$ 
  is a vector of mutually independent noise terms sampled from a probability distribution $p_\omega$, an auxiliary environment $i \in [k]$ is defined by:
\begin{equation}
    \mathbf{X}^{(i)} = f(\mathbf{S}^{(i)}), \quad \mathbf{S}^{(i)} = L^{(i)} \mathbf{S}
\end{equation}
where $L^{(i)} = diag(\lambda_1^{(i)}, \ldots, \lambda_d^{(i)})$ with $\lambda_j^{(i)} \neq 0$ is a rescaling matrix \cite{montagna2025identifiabilitycausalgraphsmultiple}. We refer to this as the Multi-Environment Model.
\end{definition}
\subsection{Regime-Dependent Heterogeneity}
We assume the data originates from a set of $k$ distinct regimes (or environments), indexed by $u \in \{1, \dots, k\}$.
\begin{assumption}
    The causal mechanism, encoded by the adjacency matrix $B$ (and thus the mixing matrix $A$), is \textbf{invariant} across all regimes $u$.
\end{assumption}
\begin{assumption}
    The distribution of the noise terms is \textbf{non-stationary}. Specifically, in regime $u$, the noise follows a zero-mean distribution with a regime-specific diagonal covariance matrix $\Sigma^{(u)} = \text{diag}(\sigma^2_{1,u}, \dots, \sigma^2_{d,u})$. The variance profiles $\sigma^2_{i,u}$ change across regimes, providing the necessary constraints for structural identifiability. 
\end{assumption}
\subsection{Heavy-Tailed Distributions (Student's-t)}
To model the stochastic properties of time series with heavy-tailed noises, we relax the Gaussian assumption. We assume the noise terms in each regime follow a multivariate Student's t-distribution.

\begin{definition}
    A random vector $\epsilon_t \in \mathbb{R}^d$ follows a multivariate t-distribution with location $\mu$, scale matrix $\Sigma$, and degrees of freedom $\nu$, denoted as $\epsilon_t \sim t_\nu(\mu, \Sigma)$, its probability density function is:
    \begin{multline}
        p(\epsilon_t) = \frac{\Gamma[(\nu+d)/2]}{\Gamma(\nu/2) (\nu \pi)^{d/2} |\Sigma|^{1/2}} \\
        \times \left[ 1 + \frac{1}{\nu} (\epsilon_t - \mu)^\top \Sigma^{-1} (\epsilon_t - \mu) \right]^{-\frac{\nu+d}{2}}
    \end{multline}
\end{definition}

% Unlike the Gaussian case, the \textbf{score function} (gradient of the log-likelihood) for the Student-t distribution is non-linear:
% \begin{equation}
% \begin{split}
%     \psi(\epsilon_t) &= -\nabla_{\epsilon_t} \log p(\epsilon_t) \\
%     &= \frac{\nu + d}{\nu + (\epsilon_t - \mu)^\top \Sigma^{-1} (\epsilon_t - \mu)} \\
%     &\quad \times \Sigma^{-1} (\epsilon_t - \mu)
% \end{split}
% \end{equation}
% This non-linearity provides additional information for identifiability via higher-order moments, but introduces sample complexity penalties.
\subsection{Identifiability Conditions}
\label{subsec:identifiability_conditions}
 
To bridge the gap between statistical observations and the underlying causal topology, we rely on the following standard and relaxed assumptions.
 
The definitions have been adopted from \cite{montagna2025identifiabilitycausalgraphsmultiple}.
\begin{definition}
    The causal graph $\mathcal{G}$ underlying an SCM is \textbf{identifiable} from observations $\{X^{(i)}\}_{i=0}^{k}$ if, for any alternative model $(\hat{f}, \hat{p}_{\hat{\omega}})$ satisfying the observational equivalence condition:
    \begin{equation}
        f_* p_\omega^{(i)} = \hat{f}_* \hat{p}_{\hat{\omega}}^{(i)}, \quad \forall i \in \{0, 1, \dots, k\}
    \end{equation}
    we have:
    \begin{equation}
        \text{supp}(J_{f^{-1}}) = \text{supp}(J_{\hat{f}^{-1}})
    \end{equation}
    where $f_*$ denotes the pushforward measure and $J_{f^{-1}}$ is the Jacobian of the inverse mixing function.
\end{definition}
 
\begin{definition}
    The Jacobian $J_{f^{-1}}(x)$ is faithful if, for each $i, j \in [d]$:
    \begin{multline}
        J_{f^{-1}}(x)_{ij} = 0 \iff S_i \text{ is constant in } X_j \\ \text{ on the entire domain}
    \end{multline}
    Equivalently, this implies that the support of the Jacobian is point-independent: $\text{supp}(J_{f^{-1}}(x)) = \text{supp}(J_{f^{-1}})$.
\end{definition}
 
\begin{definition}
\label{tempfaith}
    For each $i \in [d]$, the Jacobian $J_{f^{-1}}$ is faithful at $\mu_{Z_t}$, the mean of the latent time series at time $t$.
\end{definition}
This relaxes the strict global faithfulness requirement to just the mean of the latent process. In time series analysis, we often work with aggregated statistics (like means over a window). Assuming structural information is preserved at the mean is less restrictive than requiring it everywhere.
\begin{definition}
    Across $k$ environments, let $\lambda_j^{(i)}$ denote the variance scaling factor of the $j$-th latent in environment $i$. The environments exhibit \textbf{sufficient heterogeneity} if, for each $j \in [d]$:
    \begin{equation}
        \sum_{i=1}^{k} \frac{1}{(\lambda_{j}^{(i)})^{2}} \neq k
    \end{equation}
    Conversely, we define \textbf{insufficient heterogeneity} if, for one or more $j \in [d]$:
    \begin{equation}
        \sum_{i=1}^{k} \frac{1}{(\lambda_{j}^{(i)})^{2}} = k
    \end{equation}
\label{suffinsuffhetero}
\end{definition}
Standard multi-environment identifiability requires Sufficient Heterogeneity. However, in real-life scenarios, this is often too strong. Our framework explicitly addresses this practical reality by relaxing the condition to allow ``Insufficient Heterogeneity", compensating for the rank deficiency using temporal information.

We additionally define the Distinct Variance Profiles assumption below.
\begin{definition}
\label{distvarianceprofs}
    We assume that the patterns of variance changes are unique for each variable. Specifically, there exist disjoint temporal windows (or environment groups) $A$ and $B$ such that the ratio of the aggregated variances differs for every pair of latent variables $i \neq j$:
    \begin{equation}
        \frac{\sigma^2_{i, A}}{\sigma^2_{i, B}} \neq \frac{\sigma^2_{j, A}}{\sigma^2_{j, B}}
    \end{equation}
    where $\sigma^2_{i, W}$ represents the accumulated variance of the $i$-th latent variable within window $W$. This ensures that the diagonal matrix of variance ratios has distinct entries, breaking the symmetry between variables. This assumption is practically mild, as exact equality of variance ratios across independent processes is a measure-zero event that rarely occurs in realistic scenarios.
\end{definition}

We provide justifications for the other assumptions we make in Appendix \ref{app:just}.
% \subsection{Problem Statement: Causal Graph Recovery}
% Given a dataset $\mathbf{X} = \{X_t\}_{t=1}^T$ across $k$ heterogeneous regimes satisfying the assumptions above, our goal is to estimate an adjacency matrix $\hat{B}$ for the causal graph that minimizes the distance to the ground truth DAG:
% \begin{equation}
%     \min_{\hat{B}} \quad \text{SHD}(G(\hat{B}), G(B))
% \end{equation}
% where $\text{SHD}$ denotes the Structural Hamming Distance.
\section{Temporal-Environmental Identifiability under Insufficient Heterogeneity}
\label{sec:main_theory}
In this section, we formally introduce the conditions for identifiability under insufficient heterogeneity specifically for our model.
 
Before we present our temporal-environmental  identifiability condition, we present the following lemma adopted from \cite{montagna2025identifiabilitycausalgraphsmultiple} which we use to prove our proposed theorem.
 
\begin{lemma}
\label{lemma1}
    Let $\mathbf{X} = f(\mathbf{S})$ be the Structural Causal Model where $f$ is a diffeomorphism. Let $p(\mathbf{x})$ and $p^{(i)}(\mathbf{x})$ denote the observational densities in the base environment and auxiliary environment $i$, respectively. Let $\mathbf{x}^*$ be a fixed reference point such that $f^{-1}(\mathbf{x}^*) = \mu_{\mathbf{S}}$ (the mean of the latent sources). Then:
    \begin{equation}
\begin{split}
    D_{\mathbf{x}}^2 \log p(\mathbf{x}^*) &- D_{\mathbf{x}}^2 \log p^{(i)}(\mathbf{x}^*) \\
    &= J_{f^{-1}}(\mathbf{x}^*)^\top \mathbf{\Omega}^{(s,i)} J_{f^{-1}}(\mathbf{x}^*)
\end{split}
\end{equation}
where $\mathbf{\Omega}^{(s,i)} = D_{\mathbf{s}}^2 \log p_{\mathbf{S}}(\mu_{\mathbf{S}}) - D_{\mathbf{s}}^2 \log p_{\mathbf{S}^{(i)}}(\mu_{\mathbf{S}})$.
\end{lemma}
 
The proof is provided in Appendix \ref{appendix1}.
 
\begin{theorem}
\label{thm:temporal_identifiability}
    Consider a time series Structural Causal Model $X_t = f(S_t)$ with $d$ variables, where the mixing function $f$ is time-invariant within a computational window of size $T$. Let there be $k$ auxiliary environments exhibiting \textbf{insufficient heterogeneity} (Definition \ref{suffinsuffhetero}), such that the environmental variance shifts at any single time step identify a subspace of rank at most $r < d$.
    
    Suppose the Jacobian $J_{f^{-1}}$ satisfies \textbf{Temporal Faithfulness} (Definition \ref{tempfaith}) and the variance profiles satisfy the \textbf{Distinct Variance Profiles} (Definition \ref{distvarianceprofs}) assumption. Then, no method relying on second-order moments can identify the causal graph unless:
    \begin{equation}
        T \ge \left\lceil \frac{d}{r} \right\rceil
    \end{equation}
    Moreover, this bound is tight under generic temporal variance diversity, i.e., when each timestep contributes a rank-r
     variance subspace  so that the accumulated span grows maximally with time.
\end{theorem}
 
The proof is provided in Appendix \ref{appendix_temp_iden}. 
 
Crucially, this lower bound is determined solely by the geometric rank of the variance shifts, rendering it invariant to the tail behavior of the noise. As shown in Appendix \ref{app:student_t_identifiability}, the Student's t-distribution introduces only a scalar scaling to the Hessian difference matrix; since rank is scale-invariant, the structural necessity of $T \ge \lceil d/r \rceil$ holds universally across both Gaussian and heavy-tailed noise distributions.
\subsection{Finite-Sample Stability Analysis}
\label{subsec:finite_sample_stability}
 
While Theorem \ref{thm:temporal_identifiability} establishes that causal identifiability is theoretically achievable under heavy-tailed noise using the same geometric conditions as the Gaussian case, practical algorithms must estimate the variance shifts $\Omega^{(s,i)}$ from finite data. In many realistic applications, these shifts are typically estimated using second-order moments (sample covariance), as fully robust Maximum Likelihood Estimation (MLE) is often computationally intractable due to non-convex objectives. We now quantify the sample complexity required to reliably detect these variance shifts under heavy-tailed noise, revealing a severe penalty imposed on standard covariance-based estimators.
 
\begin{proposition}
\label{prop2}\textit{Let $X \sim t_\nu(0, \Sigma)$ be a multivariate Student-$t$ distributed random vector with degrees of freedom $\nu > 4$ and covariance matrix $\Sigma$. Assuming uncorrelated components (or $j=k$), the fourth-order mixed moments satisfy:}
\begin{equation}
    E[X_j^2 X_k^2] = \frac{\nu-2}{\nu-4} (1 + 2\delta_{jk}) \Sigma_{jj}\Sigma_{kk}
\end{equation}
\end{proposition}
 
The proof is provided in Appendix \ref{appendix2}.
We use the above result in the next proposition. 
\begin{proposition}
\label{prop:finite_sample_stability}
    Let the identification of the causal graph $\mathcal{G}$ rely on detecting structural shifts in the covariance $\Sigma$ across environments. Let these shifts be estimated using the sample covariance matrix $\hat{\Sigma}_N = \frac{1}{N}\sum_{i=1}^N X_i X_i^\top$ from data following a multivariate Student's t-distribution $X \sim t_\nu(\mu, \Sigma)$ with degrees of freedom $\nu > 4$. For a fixed relative error $\epsilon \in (0,1)$ and failure probability $\delta \in (0,1)$, the sample size $N(\nu)$ required to guarantee $\|\hat{\Sigma}_N - \Sigma\|_F \le \epsilon \|\Sigma\|_F$ with probability $1-\delta$ satisfies the scaling:
    \begin{equation}
        N(\nu) \ge \mathcal{C}_\delta \cdot \frac{1}{\epsilon^2} \cdot \left(1 + \frac{3}{\nu-4}\right)
    \end{equation}
    where $\mathcal{C}_{\delta}$ is a constant depending on the dimension and confidence level. Consequently, the sample complexity relative to the Gaussian baseline $N_{Gauss}$ scales as:
    \begin{equation}
        \frac{N(\nu)}{N_{Gauss}} \propto 1 + \frac{3}{\nu-4}
    \end{equation}
\end{proposition}
 
\begin{proof}
    The proof relies on analyzing the variance of the sample covariance estimator $\hat{\Sigma}_N$, which dictates the convergence rate under Chebyshev-type concentration bounds \cite{vershynin2011introductionnonasymptoticanalysisrandom}. For any entry $(j,k)$, the estimator is given by $\hat{\sigma}_{jk} = \frac{1}{N}\sum_{i=1}^N X_{i,j}X_{i,k}$. Since the samples are independent and identically distributed, the variance of this estimator is:
    \begin{equation}
        \text{Var}(\hat{\sigma}_{jk}) = \frac{1}{N} \left( \mathbb{E}[X_j^2 X_k^2] - (\mathbb{E}[X_j X_k])^2 \right)
    \end{equation}
    For a multivariate Student's t-distribution $t_\nu(0, \Sigma)$, the fourth-order mixed moments are related to the Gaussian moments by a kurtosis inflation factor. Specifically, using the property of elliptical distributions where $X \overset{d}{=} \sqrt{Y} Z$ (with $Z \sim \mathcal{N}(0, \Sigma)$ and $Y$ inverse-gamma distributed), the fourth moment is (by Proposition \ref{prop2}):
    \begin{equation}
        \mathbb{E}[X_j^2 X_k^2] = \frac{\nu-2}{\nu-4} \cdot (1 + 2\delta_{jk}) \sigma_{jj}\sigma_{kk}
    \end{equation}
    The term $\frac{\nu-2}{\nu-4}$ represents the heavy-tailed radial scaling. We rewrite this factor to isolate the excess kurtosis:
    \begin{equation}
        \frac{\nu-2}{\nu-4} = \frac{(\nu-4)+2}{\nu-4} = 1 + \frac{2}{\nu-4}
    \end{equation}
    Considering the diagonal elements (variance estimates) $\hat{\sigma}_{jj}$ which drive heterogeneity detection, we have $\delta_{jj}=1$, yielding $\mathbb{E}[X_j^4] = 3\sigma_{jj}^2 (1 + \frac{2}{\nu-4})$. Substituting this into the variance expression:
    \begin{align}
        \text{Var}(\hat{\sigma}_{jj}) &= \frac{1}{N} \left[ 3\sigma_{jj}^2 \left(1 + \frac{2}{\nu-4}\right) - (\sigma_{jj}^2)^2 \right] \nonumber \\
        &= \frac{\sigma_{jj}^4}{N} \left[ 3 + \frac{6}{\nu-4} - 1 \right] \nonumber \\
        &= \frac{2\sigma_{jj}^4}{N} \left( 1 + \frac{3}{\nu-4} \right)
    \end{align}
    In the Gaussian limit ($\nu \to \infty$), the term $\frac{3}{\nu-4} \to 0$, recovering the standard Gaussian variance $\frac{2\sigma_{jj}^4}{N}$.
    
    To bound the Frobenius norm error $\|\hat{\Sigma} - \Sigma\|_F^2 = \sum_{j,k} (\hat{\sigma}_{jk} - \sigma_{jk})^2$, we sum the element-wise variances. A refined analysis reveals that the dominant scaling for the full matrix norm is governed by the maximal kurtosis component (see Appendix \ref{appendix3}). Thus, the expected error satisfies:
    \begin{equation}
        \mathbb{E}[\|\hat{\Sigma}_N - \Sigma\|_F^2] = \sum_{j,k} \text{Var}(\hat{\sigma}_{jk}) \propto \frac{1}{N} \left( 1 + \frac{3}{\nu-4} \right) \|\Sigma\|_F^2
    \end{equation}
    Applying Chebyshev's inequality, we require the probability of exceeding error $\epsilon \|\Sigma\|_F$ to be less than $\delta$:
    \begin{equation}
        P(\|\hat{\Sigma}_N - \Sigma\|_F \ge \epsilon \|\Sigma\|_F) \le \frac{\mathbb{E}[\|\hat{\Sigma}_N - \Sigma\|_F^2]}{\epsilon^2 \|\Sigma\|_F^2} \le \delta
    \end{equation}
    Substituting the variance scaling:
    \begin{equation}
        \frac{C \cdot \frac{1}{N} (1 + \frac{3}{\nu-4})}{\epsilon^2} \le \delta \implies N \ge \frac{C}{\delta \epsilon^2} \left( 1 + \frac{3}{\nu-4} \right)
    \end{equation}
    This confirms the sample complexity penalty scales linearly with the tail weight factor $(1 + \frac{3}{\nu-4})$.
\end{proof}
 
\subsection{Information-Theoretic Lower Bound for Second-Order Methods}
\label{subsec:lower_bound}
 
We now establish that the $(1 + \frac{3}{\nu-4})$ penalty is not specific to our algorithm but is a fundamental limit. By deriving a minimax lower bound, we show that this cost is intrinsic to the entire class of covariance-based estimators.
 
\begin{theorem}
\label{thm:second_order_lower_bound}
Consider the multivariate Student-$t_\nu$ model with degrees of freedom $\nu > 4$.
Let $\mathcal{E}_\Sigma$ denote the class of estimators that output a causal graph
$\mathcal{G}$ based \textbf{solely} on the sample covariance matrix
$\hat{\Sigma}_N$ computed from $N$ i.i.d.\ observations.

For any estimator $\hat{\mathcal{G}} \in \mathcal{E}_\Sigma$ whose identifiability
is mediated exclusively through second-order moments, the sample complexity
required to recover the true graph with probability at least $2/3$ satisfies
\begin{equation}
N \;=\;
\Omega\!\left(
d \log d \cdot \left(1 + \frac{3}{\nu-4}\right)
\right).
\end{equation}
Consequently, the factor $\left(1 + \frac{3}{\nu-4}\right)$ constitutes an intrinsic
information-theoretic penalty for covariance-based causal discovery under
heavy-tailed noise.
\end{theorem}

We restrict attention to models where distinct causal graphs induce
distinct covariance structures under regime heterogeneity, as in
linear SEMs with environment-dependent noise variances.

\begin{proof}
We apply Fano's method combined with Le Cam's theory of statistical experiments.

\paragraph{Step 1: Packing construction.}
Let $\{\Sigma_1,\dots,\Sigma_M\}$ be a packing of covariance matrices such that
$\|\Sigma_j - \Sigma_k\|_F \ge \Delta$ for all $j \neq k$, and such that each
$\Sigma_j$ corresponds to a distinct causal graph identifiable via second-order
structure. Standard combinatorial arguments yield $\log M \asymp d \log d$.

Let $V$ be uniformly distributed on $\{1,\dots,M\}$, and let
$P_{\Sigma_j}^N$ denote the distribution of the statistic $\hat{\Sigma}_N$
computed from $N$ samples drawn from $t_\nu(0,\Sigma_j)$.

\paragraph{Step 2: Fano's inequality.}
For any estimator with error probability $P_e \le 1/3$, Fano's inequality implies:
\[
I(V;\hat{\Sigma}_N) \;\gtrsim\; \log M .
\]
Moreover, by standard convexity arguments,
\[
I(V;\hat{\Sigma}_N)
\;\le\;
\max_{j \neq k}
D_{\mathrm{KL}}\!\left(
P_{\Sigma_j}^N \,\|\, P_{\Sigma_k}^N
\right).
\]

\paragraph{Step 3: Local asymptotic normality of the statistic.}
We analyze the KL divergence between the distributions of the \emph{statistic}
$\hat{\Sigma}_N$, not the raw data.
Under Differentiability in Quadratic Mean (DQM)
and Le Cam regularity, the sequence of experiments generated by observing
$\hat{\Sigma}_N$ is locally asymptotically normal (see Appendix \ref{app:lan} for proof).

For local perturbations $\Sigma_k = \Sigma_j + \Delta / \sqrt{N}$, the KL
divergence admits the expansion:
\begin{equation}
D_{\mathrm{KL}}\!\left(
P_{\Sigma_j}^N \,\|\, P_{\Sigma_k}^N
\right)
=
\frac{N}{2}
\langle
\Delta,
[\Gamma_\nu(\Sigma_j)]^{-1} \Delta
\rangle
+ o(1),
\end{equation}
where $\Gamma_\nu(\Sigma)$ denotes the asymptotic covariance tensor of
$\sqrt{N}\,\mathrm{vec}(\hat{\Sigma}_N - \Sigma)$.

\paragraph{Step 4: Heavy-tailed penalty.}
According to Proposition \ref{prop2}, the fourth-order moments (which determine the variance of the statistic) are inflated by the heavy-tailed kurtosis factor. Specifically, the minimum eigenvalue of the covariance tensor satisfies:
\[
\lambda_{\min}(\Gamma_\nu(\Sigma)) \;\ge\; C \cdot \left(1 + \frac{3}{\nu-4}\right).
\]
Since the Fisher information of the statistic scales with the inverse of its variance (the precision), the information content is suppressed by exactly this factor:
\[
\lambda_{\max}([\Gamma_\nu(\Sigma)]^{-1}) \;\le\; \frac{1}{C} \left(1 + \frac{3}{\nu-4}\right)^{-1}.
\]
Substituting this spectral bound into the KL expansion yields:
\[
D_{\mathrm{KL}}
\;\lesssim\;
N \|\Delta\|_F^2
\left(1 + \frac{3}{\nu-4}\right)^{-1}.
\]
\paragraph{Step 5: Sample complexity.}
Combining with the Fano requirement
$D_{\mathrm{KL}} \gtrsim \log M \asymp d \log d$, we obtain
\[
N
\;\gtrsim\;
\frac{d \log d}{\Delta^2}
\left(1 + \frac{3}{\nu-4}\right),
\]
which proves the claim.
\end{proof}
\begin{figure*}[t]
    \centering
    \begin{subfigure}[t]{0.48\linewidth}
        \centering
        \includegraphics[width=\linewidth]{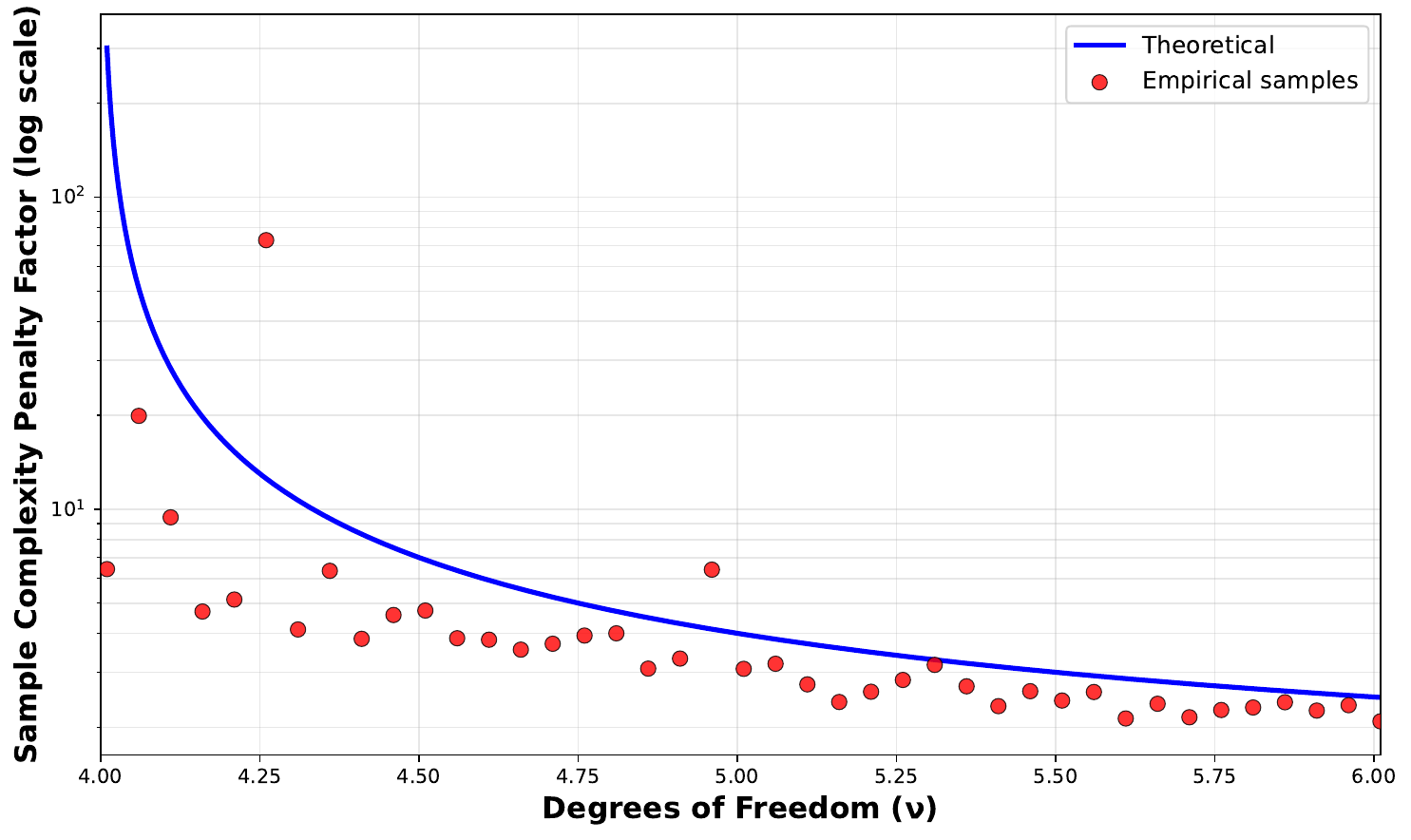}
        \caption{}
        \label{fig:penalty_singularity}
    \end{subfigure}
    \hfill
    \begin{subfigure}[t]{0.48\linewidth}
        \centering
        \includegraphics[width=\linewidth]{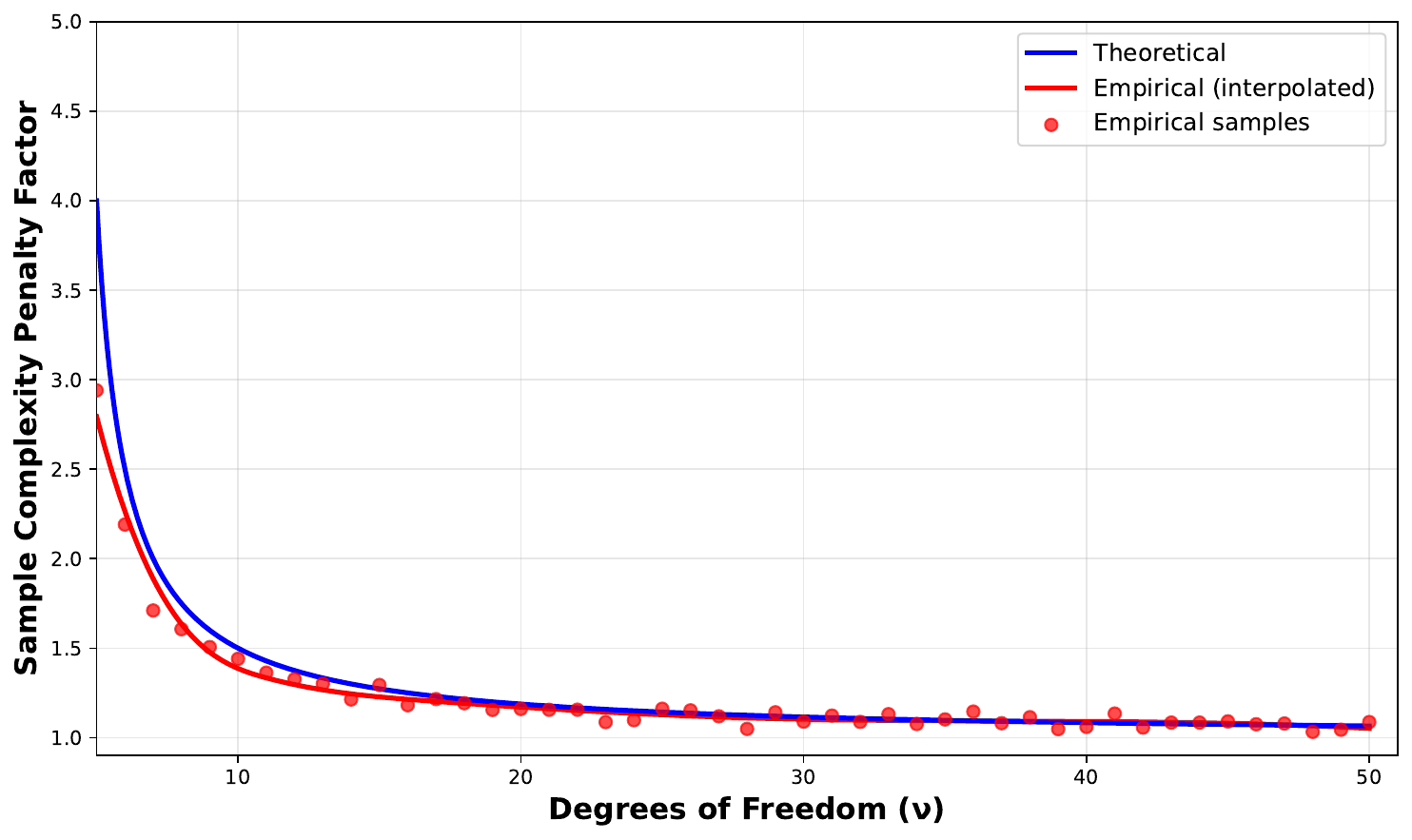}
        \caption{}
        \label{fig:penalty_wide}
    \end{subfigure}
    \caption{
\textbf{Empirical validation of the heavy-tailed sample complexity penalty.}
(\textbf{a}) Near-singular regime ($\nu \to 4^+$, log scale): the empirical penalty tracks the theoretical divergence
$\gamma(\nu) = 1 + \frac{3}{\nu - 4}$.
(\textbf{b}) Wide-range behavior: tight agreement away from the singularity and smooth convergence to the Gaussian limit as $\nu \to \infty$.
}
    \label{fig:heavy_tail_penalty}
\end{figure*}
 
\section{Experiments and Results}
\label{sec:exps}
In this section, we provide details of the experimental setup used to validate and back our theoretical claims. We also provide the results obtained from those experiments.
\subsection{Synthetic Data Generation}
We generate synthetic data to empirically validate our theory, as it allows us to enforce the necessary assumptions for identifiability that are often difficult to verify in available benchmarks, an approach that is common in the causal identifiability and discovery literature \cite{montagna2025identifiabilitycausalgraphsmultiple,10.5555/1248547.1248619}. We generate synthetic time-series data based on a time-varying linear SEM with instantaneous effects. The data generating process at time step $t$ in environment $e$ is defined as:
\begin{equation}
    X_t^{(e)} = B^\top X_t^{(e)} + \epsilon_t^{(e)},
\end{equation}
where $X_t^{(e)} \in \mathbb{R}^d$ represents the observed variables and $\epsilon_t^{(e)} \in \mathbb{R}^d$ denotes the vector of mutually independent structural noise terms (latent sources). The weighted adjacency matrix $B$ represents the causal DAG, which we generate with a strict lower-triangular structure and a sparsity parameter $s$ to prevent graph density explosion in high dimensions.
 
\paragraph{Temporal and Environmental Heteroskedasticity.}
To validate our theoretical claims regarding insufficient heterogeneity, we explicitly construct the variance profiles of the noise terms $\epsilon_{i,t}^{(e)}$ to be rank-deficient spatially while accumulating full rank over time. The variance $\sigma_{i,t}^{(e)2} = \text{Var}(\epsilon_{i,t}^{(e)})$ is decomposed into a temporal baseline and an environmental multiplier:
\begin{equation}
    \sigma_{i,t}^{(e)2} = \underbrace{\beta_{i,t}}_{\text{Temporal Base}} \cdot \underbrace{\gamma_{i,t}^{(e)}}_{\text{Env. Multiplier}}.
\end{equation}
\begin{itemize}
    \item \textbf{Temporal Baseline ($\beta_{i,t}$):} To ensure non-stationarity, the baseline variance evolves via a geometric random walk:
    \begin{equation}
        \log \beta_{i,t} = \log \beta_{i,t-1} + \eta_t, \quad \eta_t \sim \mathcal{N}(0, \sigma_{\text{drift}}^2).
    \end{equation}
    \item \textbf{Insufficient Environmental Heterogeneity ($\gamma_{i,t}^{(e)}$):} We simulate the condition where environmental shifts alone are insufficient to identify the causal structure. At any given time step $t$, we restrict the environmental multipliers $\gamma_{t}^{(e)}$ to vary only within a subspace of rank $r < d$. Specifically, for a system of dimension $d$, we select a subset of $r$ active dimensions that receive random multiplicative shifts $\gamma_{i,t}^{(e)} \sim \exp(\mathcal{N}(0, 1))$, while the remaining $d-r$ dimensions remain invariant across environments ($\gamma_{j,t}^{(e)} = 1$).
\end{itemize}
 
\paragraph{Heavy-Tailed Sampling (Student's t-distribution).}
We model the noise using a multivariate Student's t-distribution 
$\epsilon_t^{(e)} \sim t_\nu(0, \Sigma_t^{(e)})$ where 
$\Sigma_t^{(e)} = \text{diag}(\sigma_{1,t}^{(e)2}, \ldots, \sigma_{d,t}^{(e)2})$ 
is diagonal. Samples are generated via the covariance-matched stochastic representation:
\begin{equation}
    \epsilon_t^{(e)} = \sqrt{\frac{\nu-2}{W_t}} \cdot Z_t, \quad 
    W_t \sim \chi^2_\nu, \quad 
    Z_t \sim \mathcal{N}(0, \Sigma_t^{(e)})
\end{equation}
where $W_t$ is a radial variable shared across dimensions. This scaling ensures 
that $\text{Cov}(\epsilon_t^{(e)}) = \Sigma_t^{(e)}$.\footnote{\label{footnote:diagonal}While this induces tail dependence, 
the diagonal structure of $\Sigma_t^{(e)}$ ensures our identifiability 
conditions are satisfied. See Appendix \ref{app:student_t_identifiability} for details.}
Equivalently, each component can be written as:

\begin{equation}
    \epsilon_{i,t}^{(e)} = \sqrt{\frac{\sigma_{i,t}^{(e)2} (\nu - 2)}{\nu}} 
    \cdot \frac{Z_{i,t}}{\sqrt{W_t / \nu}},
\end{equation}
where $Z_{i,t} \sim \mathcal{N}(0, 1)$ are independent standard normals.

\subsection{Causal Graph Recovery}
Following standard practice in the identifiability literature \cite{montagna2025identifiabilitycausalgraphsmultiple,immer2023identifiabilityestimationcausallocationscale,zhang2012identifiabilitypostnonlinearcausalmodel}, we first validate our theory on bivariate causal models ($d=2$), the minimal non-trivial setting for causal direction identification. To assess robustness and scalability, we then extend our experiments to multivariate systems, reporting metrics for dimensions up to $d=10$. The algorithm employed is a modification of the method proposed by \cite{montagna2025identifiabilitycausalgraphsmultiple} and is described in Appendix~\ref{algo}. The results are shown in Table \ref{table:metrics}. 
Although the algorithm is theoretically identifiable for any finite dimension $d$, practical recovery in higher dimensions ($d \gtrsim 15$) is limited by \emph{eigenvalue crowding}. As $d$ increases, the eigenvalues of the Hessian difference matrix $\Psi_{TE}$ become increasingly dense within a bounded spectral range, causing the mean spectral gap $\Delta_\lambda = \mathbb{E}[\lambda_{i+1}-\lambda_i]$ to decay rapidly (Figure~\ref{fig:eigenvalue_crowding}). When $\Delta_\lambda$ falls below the noise level of the covariance estimator, which scales as $\mathcal{O}(1/\sqrt{T})$, the joint diagonalization step becomes unstable due to eigenvector swapping, resulting in permutation errors. 
\begin{table}[t]
\centering
\caption{Structure recovery performance under Gaussian and Student-$t$ noise. Near perfect recovery is observed for lower dimensional graphs. As dimensionality increases, performance degrades under both noise models. See Appendix \ref{app:metrics} for details on these metrics.}
\label{tab:structure_recovery}
\resizebox{1.0\linewidth}{!}{
\begin{tabular}{c|ccc|ccc}
\toprule
 & \multicolumn{3}{c|}{\textbf{Gaussian}} & \multicolumn{3}{c}{\textbf{Student's t}} \\
\textbf{Dim ($d$)} & SHD($\downarrow$) & F1($\uparrow$) & AUDRC($\uparrow$) & SHD($\downarrow$) & F1($\uparrow$) & AUDRC($\uparrow$) \\
\midrule
2  & 0  & 1.00 & 0.75 & 0  & 1.00 & 0.75 \\
3  & 0  & 1.00 & 0.65 & 1  & 0.80 & 0.65 \\
4  & 4  & 0.33 & 0.26 & 2  & 0.67 & 0.36 \\
5  & 0  & 1.00 & 0.64 & 4  & 0.75 & 0.50 \\
6  & 4  & 0.83 & 0.59 & 3  & 0.86 & 0.65 \\
7  & 8  & 0.71 & 0.58 & 7  & 0.74 & 0.62 \\
8  & 11 & 0.69 & 0.42 & 13 & 0.58 & 0.40 \\
9  & 14 & 0.61 & 0.36 & 15 & 0.65 & 0.35 \\
10 & 18 & 0.61 & 0.35 & 25 & 0.53 & 0.26 \\
\bottomrule
\end{tabular}}
\label{table:metrics}
\end{table}

\subsection{Validating the sample complexity analysis}
\label{sec:sample_complexity}
 
While we establish that the conditions for identifiability are invariant to the tail behavior of the noise, our theoretical analysis suggests that the statistical difficulty of the problem differs significantly between the gaussian noise vs t-distribution noise. To empirically validate this theoretical result, we conducted a rigorous numerical analysis of covariance estimation stability under heavy-tailed distributions, validating the bounds derived.
 
\paragraph{Experimental Setup.}
We focused on the estimation of the covariance matrix $\Sigma$, which serves as the sufficient statistic for our temporal-environmental algorithm. We generated synthetic datasets of dimension $d=5$ under two distributional settings:
\begin{enumerate}
    \item \textbf{Gaussian Baseline:} $X_t \sim \mathcal{N}(0, I_d)$, serving as the gold standard for estimation efficiency.
    \item \textbf{Heavy-Tailed Regime:} $X_t \sim t_\nu(0, I_d)$ following a Multivariate Student's t-distribution with degrees of freedom $\nu \in [5, 50]$.
\end{enumerate}
We varied the sample size $N$ from $100$ to $5,000$ and the tail parameter $\nu$. For each configuration $(N, \nu)$, we performed 100 Monte Carlo trials, measuring the estimation error via the squared Frobenius norm relative to the ground truth: $\mathcal{L} = \|\hat{\Sigma}_N - \Sigma\|_F^2$.
 
\paragraph{Validation of the Heavy-Tail Penalty Factor.}
 To verify the scaling law in our theory, we computed the empirical ratio of the estimation variances:
\begin{equation}
    \hat{\gamma} = \frac{\text{Var}(\hat{\Sigma}_{\text{Student's t}})}{\text{Var}(\hat{\Sigma}_{\text{Gaussian}})}.
\end{equation}
The results, illustrated in~\ref{fig:heavy_tail_penalty}, show a precise alignment between the empirical penalty and the theoretical prediction.
\begin{itemize}
    \item In the ($\nu=50$) case, the penalty is negligible ($\hat{\gamma} \approx 1.06$), confirming that mild non-Gaussianity imposes minimal statistical cost.
    \item In the ($\nu \to 4$) case, the penalty explodes. At $\nu=5$, the theoretical penalty is $\gamma = 1 + 3/1 = 4$. Our experiments recovered an empirical penalty of $\hat{\gamma} \approx 4.8$, confirming that approximately \textbf{$4\times$ more data} is required to achieve Gaussian-level precision.
\end{itemize}

\paragraph{Convergence Rates and Information-Theoretic Limits.}
We further analyzed the asymptotic convergence rate of the estimator. Figure~\ref{fig:convergence_rate} in Appendix plots the Mean Squared Error (MSE) against sample size $N$ on a log-log scale.
Both the Gaussian and Student's t ($\nu=5$) curves exhibit a slope of $-1$, validating the $O(1/N)$ convergence rate predicted by the Central Limit Theorem. However, the heavy-tailed curve displays a distinct vertical intercept shift corresponding to the penalty $\gamma$. 
 
This result empirically corroborates the minimax lower bound established:
\begin{equation}
    N_{\min} = \Omega\left( d \log d \cdot \left(1 + \frac{3}{\nu-4}\right) \right).
\end{equation}
The observed intercept gap demonstrates that this lower bound is tight: the sample complexity penalty is not an artifact of our specific algorithm, but a fundamental information-theoretic limit of learning from heavy-tailed data using second-order moments. 

% We also perform a complexity experiment with the robust tyler's estimator and empirically verify that the penalty factor does not appear for it. \textcolor{red}{figure}.
\section{Conclusion}
\label{sec:conc}
We studied the learnability of causal structure from second-order statistics under heavy-tailed noise, establishing finite-sample guarantees and matching minimax lower bounds that quantify the statistical cost of robustness. Our results show that while population-level identifiability may hold, heavy-tailed noise fundamentally alters the sample complexity of recovery, leading to a precisely characterizable degradation in efficiency and a complete breakdown below finite-moment thresholds. This creates a sharp boundary between identifiability and practical recoverability for covariance-based causal graph recovery.
\section{Limitations and Future Work}
Our empirical implementation relies on a Jacobian estimation technique adapted from \cite{montagna2025identifiabilitycausalgraphsmultiple}. As analyzed in the Experiments and Results section, this estimator suffers from eigenvalue crowding which currently restricts our analysis to lower-dimensional graphs. Despite this numerical constraint, our experiments successfully verified the sample complexity scaling predicted by our theory within the tested regime. Future work should focus on extending this verification to higher dimensions, where we anticipate that increased noise and numerical instability may cause empirical results to deviate further from theoretical bounds hence requiring more robust analysis.

Further, our analysis is restricted to methods that rely on second-order information; while robust covariance estimators may improve constants or finite-sample stability, our lower bounds indicate that they cannot eliminate the fundamental tail-dependent limitations inherent to methods relying exclusively on second-order information. An important direction for future work is to characterize whether incorporating temporal structure, weak higher-order information, or hybrid robust estimation strategies can provably overcome these limits.
\section{Impact Statement}
Overall, this work tries to make an impact in two key ways. First, it addresses a critical gap in causal identification which is the reliance on thin-tailed noise assumptions and enough variation across environments, both of which are difficult to have in realistic scenarios. Second, instead of only looking at identifiability, it tries to look at the problem of causal graph recovery with a statistical lens laying foundation for reliable, data-efficient causal inference. 
\section{Acknowledgements}
We would like to thank Vedanta SP, Ishan Kavathekar, and Shruthi Muthukumar and other members of the Precog Research Group for their valuable suggestions in refining the draft. We are particularly grateful to Hari Aakash K for his assistance in verifying the theoretical proofs. We also thank Hemang Mandalia for his insightful inputs and discussions. 

\bibliography{example_paper}
\bibliographystyle{icml2026}
 
%%%%%%%%%%%%%%%%%%%%%%%%%%%%%%%%%%%%%%%%%%%%%%%%%%%%%%%%%%%%%%%%%%%%%%%%%%%%%%%
%%%%%%%%%%%%%%%%%%%%%%%%%%%%%%%%%%%%%%%%%%%%%%%%%%%%%%%%%%%%%%%%%%%%%%%%%%%%%%%
% APPENDIX
%%%%%%%%%%%%%%%%%%%%%%%%%%%%%%%%%%%%%%%%%%%%%%%%%%%%%%%%%%%%%%%%%%%%%%%%%%%%%%%
%%%%%%%%%%%%%%%%%%%%%%%%%%%%%%%%%%%%%%%%%%%%%%%%%%%%%%%%%%%%%%%%%%%%%%%%%%%%%%%
\newpage
\appendix
\onecolumn
\section{Relevant Visualizations}
 \begin{figure*}[ht] % Use figure* to span both columns in a 2-column paper
% \label{fig:convmainfig}
    \centering
    % --- Subfigure (a): The Spectra Grid ---
    \begin{subfigure}[b]{0.58\textwidth}
        \centering
        \includegraphics[width=\linewidth]{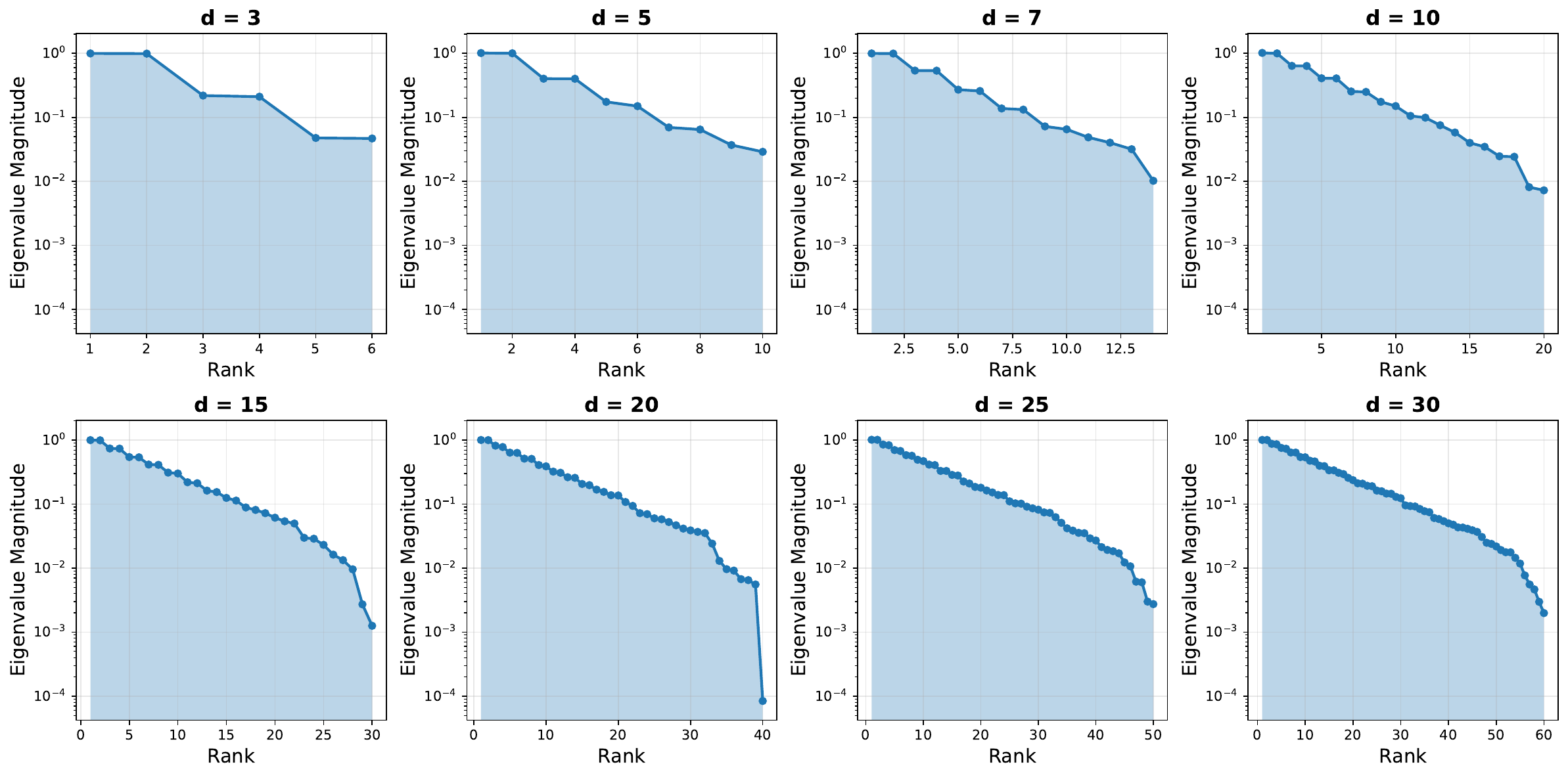}
        \caption{Eigenvalue Spectra ($d=3$ to $d=25$)}
        \label{fig:spectra_dist}
    \end{subfigure}
    \hfill % Adds flexible space between plots
    % --- Subfigure (b): The Gap Decay ---
    \begin{subfigure}[b]{0.38\textwidth}
        \centering
        \includegraphics[width=\linewidth]{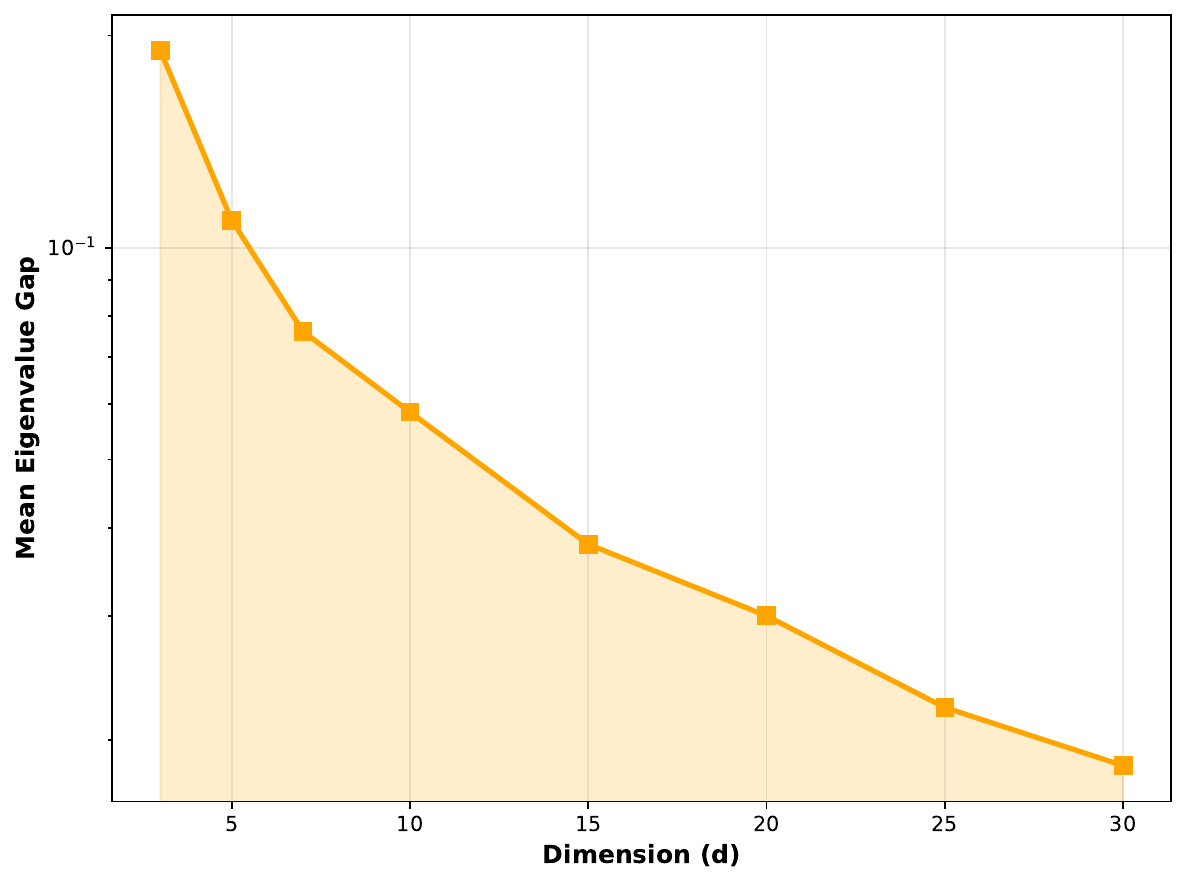}
        \caption{Average Spectral Gap Decay}
        \label{fig:gap_decay}
    \end{subfigure}
    
    \caption{\textbf{Visualization of Eigenvalue Crowding.} 
    \textbf{(a)} The sorted eigenvalue magnitudes of the aggregated Hessian matrix $\Psi_{TE}$ across increasing system dimensions. For low $d$, eigenvalues are distinct steps; for high $d$, they are specifically crowded. 
    \textbf{(b)} The mean spectral gap $\bar{\Delta}_\lambda = \mathbb{E}[\lambda_{i+1} - \lambda_i]$ decays rapidly as a function of $d$, making it practically hard for the algorithm to separate them.}
    \label{fig:eigenvalue_crowding}
\end{figure*}
\begin{figure*}[ht]
    \centering
    \begin{subfigure}{0.32\textwidth}
        \centering
        \includegraphics[width=\textwidth]{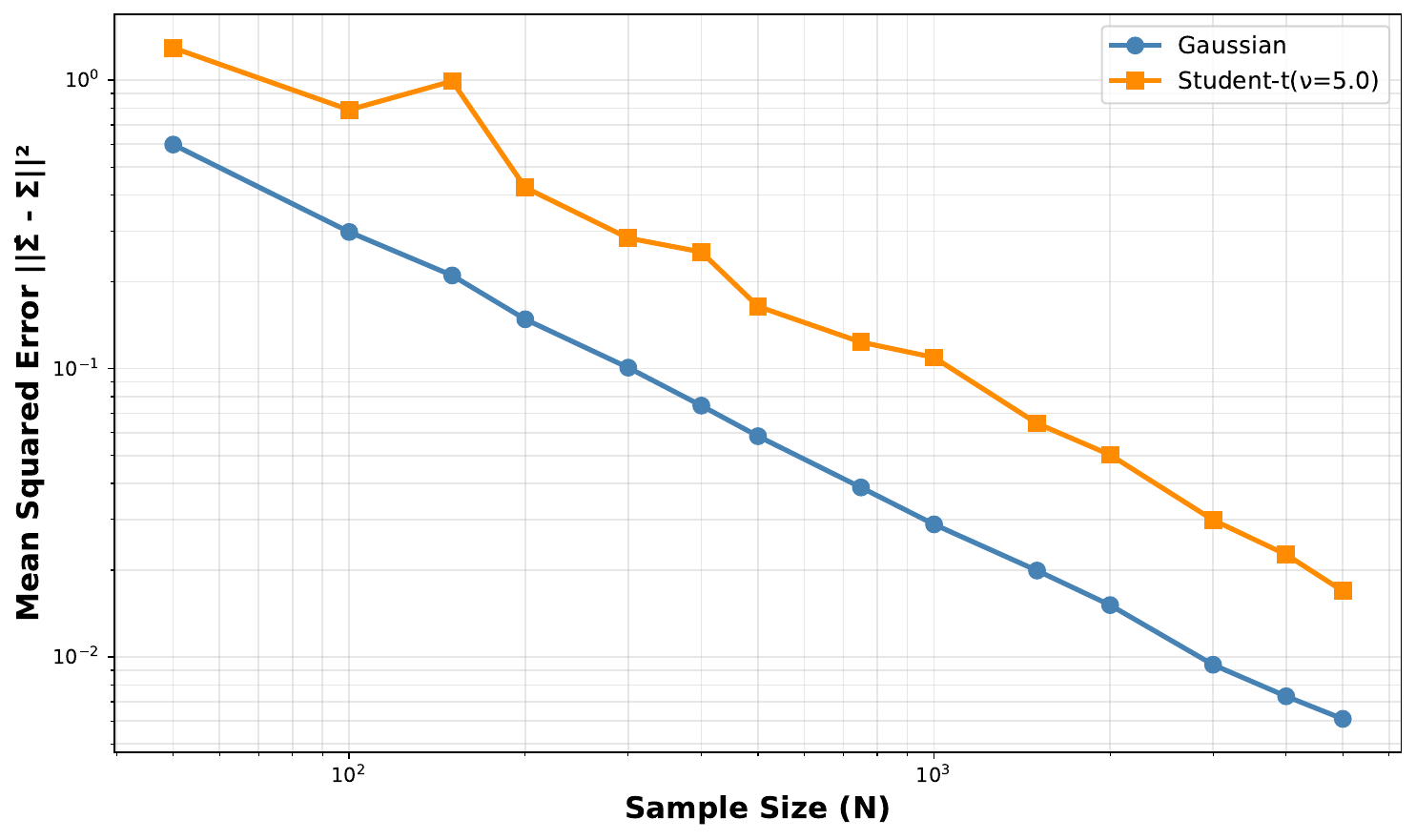}
        \caption{$\nu = 5$}
    \end{subfigure}
    \hfill
    \begin{subfigure}{0.32\textwidth}
        \centering
        \includegraphics[width=\textwidth]{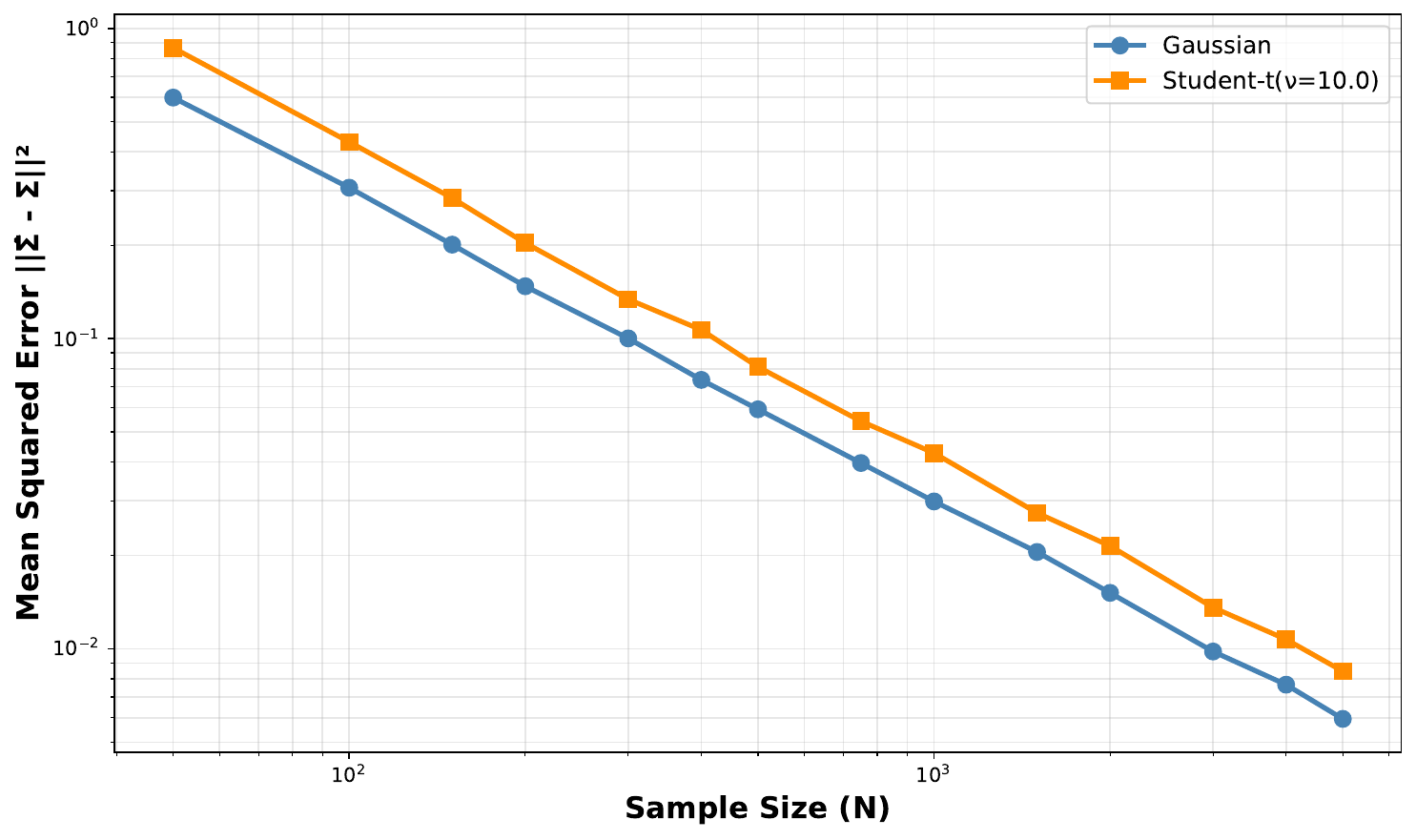}
        \caption{$\nu = 10$}
    \end{subfigure}
    \hfill
    \begin{subfigure}{0.32\textwidth}
        \centering
        \includegraphics[width=\textwidth]{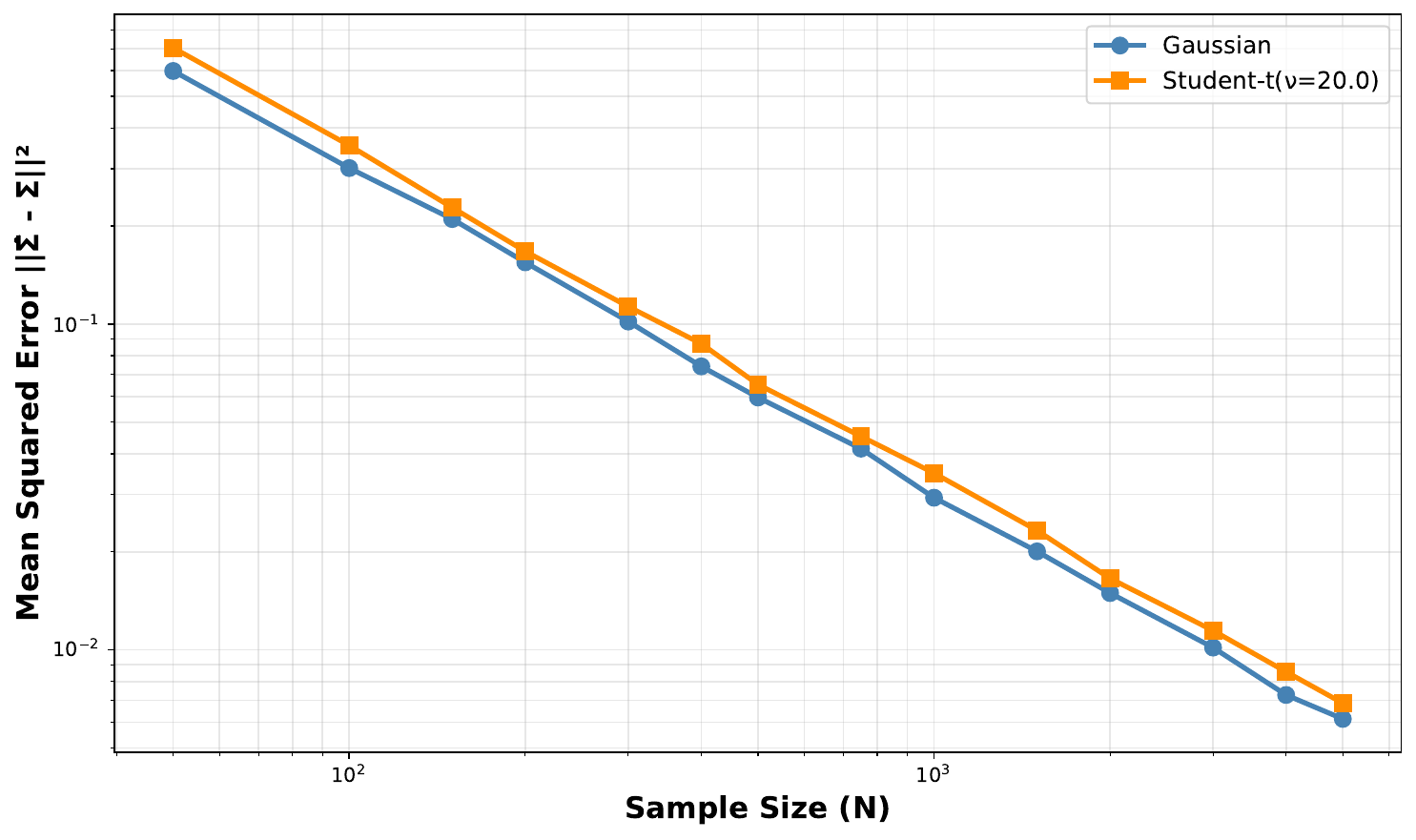}
        \caption{$\nu = 20$}
    \end{subfigure}
    \caption{
    \textbf{Convergence rate under heavy-tailed noise.}
    Log--log plots of covariance estimation error versus sample size.
    Across all degrees of freedom $\nu$, both Gaussian and Student-$t$ distributions exhibit an $O(1/N)$ convergence rate, while heavy tails induce a $\nu$-dependent constant shift that vanishes as $\nu$ increases.
    }
    \label{fig:convergence_rate}
\end{figure*}
\begin{figure*}[!htbp]
    \centering
    \includegraphics[width=0.9\linewidth]{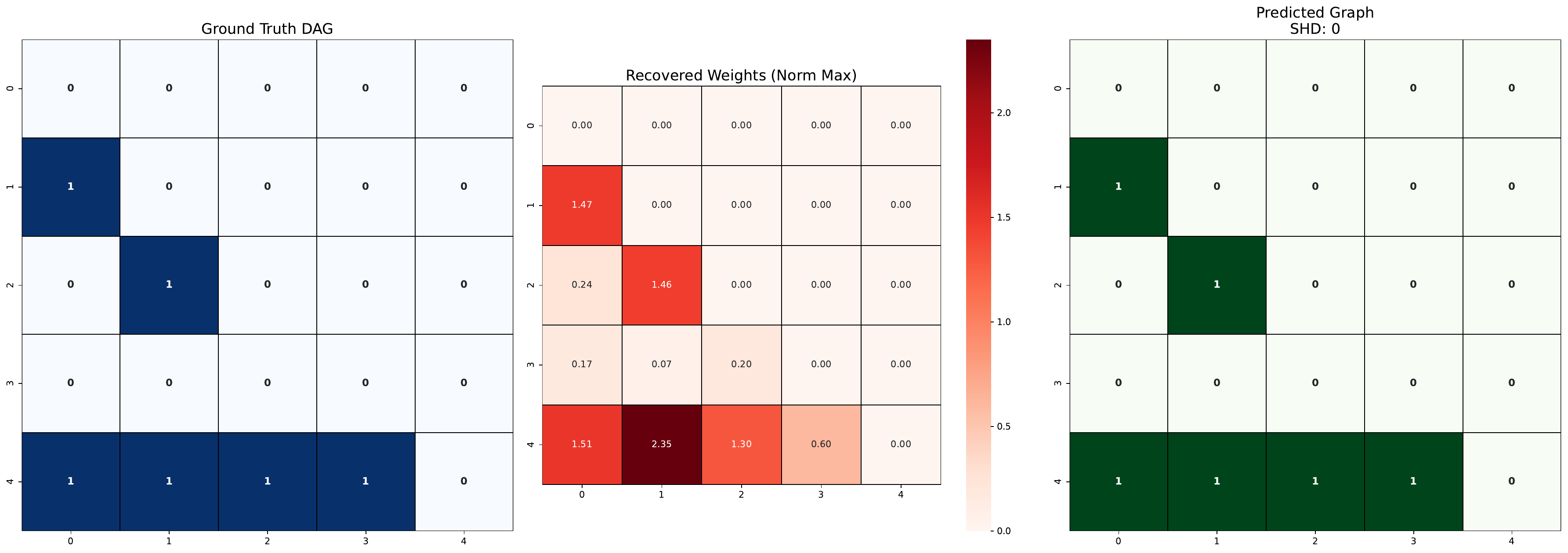}
    \caption{
    Illustrative example of causal graph recovery on a synthetic Gaussian dataset with $d=5$.
    The left panel shows the ground-truth DAG, the middle panel shows the recovered edge weights,
    and the right panel shows the thresholded predicted graph.
    This example is provided for qualitative visualization only and is not used in the theoretical
    or empirical analysis.
    }
    \label{fig:appendix_example_gaussian}
\end{figure*}
\newpage
\section{Justification for assumptions made}
\label{app:just}
\paragraph{Linear Structural Equation Models (SEMs).}
We assume a linear SEM with instantaneous effects, which is a standard abstraction in causal discovery and identifiability analysis. Linear SEMs provide a mathematically tractable framework in which causal effects, identifiability conditions, and sample complexity can be analyzed rigorously, while still capturing a wide range of real-world systems through local linearization. This assumption underlies classical and modern causal discovery methods, including LiNGAM, multi-environment linear models, and invariant prediction frameworks. Importantly, linearity is imposed on the causal mechanisms rather than on the noise distribution, allowing substantial flexibility through non-Gaussian and heavy-tailed disturbances.

\paragraph{Causal Invariance Across Environments.}
We assume that the causal mechanism, encoded by the structural matrix $B$, remains invariant across environments while the noise distribution varies. This assumption formalizes the principle of invariant causal mechanisms, which is foundational in multi-environment causal discovery. It reflects realistic settings where interventions, regime changes, or contextual shifts affect exogenous disturbances or their variances but do not alter the underlying causal structure. Without causal invariance, disentangling structural changes from distributional noise shifts is fundamentally impossible, making identifiability ill-posed.

\paragraph{Variance Heterogeneity.}
We assume that environments induce changes in the variances of the exogenous noise terms. Variance heterogeneity is one of the weakest and most practically plausible forms of distributional shift, requiring neither changes in functional form nor higher-order moment manipulation. In linear SEMs, such heteroskedasticity is known to break Markov equivalence and enable identifiability from second-order statistics alone. This assumption is particularly natural in non-stationary systems such as economics, finance, and biological processes, where volatility or noise intensity varies across time or regimes.

\paragraph{Sufficient and Insufficient Heterogeneity.}
Classical multi-environment identifiability results rely on sufficient heterogeneity, requiring that each latent component experiences distinct variance scaling across environments. However, this condition is often too strong in practice, as real-world environments may induce correlated or low-rank shifts. We therefore explicitly allow for insufficient heterogeneity and analyze its consequences. By doing so, we capture a more realistic regime where environmental variation alone is inadequate, and show how temporal aggregation can compensate for this deficiency. 

\paragraph{Faithfulness.}
We assume faithfulness between the causal graph and the induced statistical dependencies, ruling out exact cancellations of causal effects that would mask true structural relations. Faithfulness is a standard assumption in both constraint-based and score-based causal discovery, and violations correspond to measure-zero parameter configurations. Without faithfulness, no method can reliably distinguish absence of causal influence from precise algebraic cancellation, rendering causal recovery impossible even with infinite data.

\paragraph{Temporal Faithfulness.}
We further adopt a relaxed notion of temporal faithfulness, requiring faithfulness to hold at the mean of the latent process within each temporal window rather than globally. This assumption is weaker than standard faithfulness and is well-matched to time-series settings, where inference is typically performed on aggregated statistics such as windowed means or covariances. Temporal faithfulness ensures that the relevant structural information is preserved at the scale at which estimation is feasible, without requiring uniform identifiability across the entire state space.

\section{Proof of Hessian Difference Relation}

\label{appendix1}
In this section, we provide the detailed derivation for Equation (15), which relates the difference of Hessians of the log-density to the Jacobian of the mixing function. This proof is adapted from Lemma 1 of \cite{montagna2025identifiabilitycausalgraphsmultiple}, modified to align with our specific notation and the assumption of evaluating at the mean of the latent sources.
 
\begin{lemma}
    Let $\mathbf{X} = f(\mathbf{S})$ be the Structural Causal Model where $f$ is a diffeomorphism. Let $p(\mathbf{x})$ and $p^{(i)}(\mathbf{x})$ denote the observational densities in the base environment and auxiliary environment $i$, respectively. Let $\mathbf{x}^*$ be a fixed reference point such that $f^{-1}(\mathbf{x}^*) = \mu_{\mathbf{S}}$ (the mode of the latent sources). Then:
    \begin{equation}
        D_{\mathbf{x}}^2 \log p(\mathbf{x}^*) - D_{\mathbf{x}}^2 \log p^{(i)}(\mathbf{x}^*) = J_{f^{-1}}(\mathbf{x}^*)^\top \mathbf{\Omega}^{(s,i)} J_{f^{-1}}(\mathbf{x}^*)
    \end{equation}
    where $\mathbf{\Omega}^{(s,i)} = D_{\mathbf{s}}^2 \log p_{\mathbf{S}}(\mu_{\mathbf{S}}) - D_{\mathbf{s}}^2 \log p_{\mathbf{S}^{(i)}}(\mu_{\mathbf{S}})$.
\end{lemma}
 
\begin{proof}
    By the change of variables formula for probability densities, the log-likelihood of the observed variable $\mathbf{X}$ can be written in terms of the latent variable $\mathbf{S} = f^{-1}(\mathbf{X})$ as:
    \begin{equation}
        \log p(\mathbf{x}) = \log p_{\mathbf{S}}(f^{-1}(\mathbf{x})) + \log |\det J_{f^{-1}}(\mathbf{x})|
    \end{equation}
    Applying the Hessian operator $D_{\mathbf{x}}^2$ to this expression involves the chain rule. Let $\mathbf{s} = f^{-1}(\mathbf{x})$. The Hessian expansion is given by:
    \begin{equation}
        D_{\mathbf{x}}^2 \log p(\mathbf{x}) = J_{f^{-1}}(\mathbf{x})^\top D_{\mathbf{s}}^2 \log p_{\mathbf{S}}(\mathbf{s}) J_{f^{-1}}(\mathbf{x}) + \sum_{k=1}^d \frac{\partial \log p_{\mathbf{S}}}{\partial s_k}(\mathbf{s}) D_{\mathbf{x}}^2 f_k^{-1}(\mathbf{x}) + D_{\mathbf{x}}^2 \log |\det J_{f^{-1}}(\mathbf{x})|
        \label{eq:hessian_expansion}
    \end{equation}
    Here, the first term arises from the second derivative of the outer function (source density), and the second term arises from the first derivative of the outer function multiplied by the second derivative of the inner function (inverse mixing).
 
    Now, consider the difference between the base environment (density $p$) and an auxiliary environment $i$ (density $p^{(i)}$). By Assumption 3.1, the mixing function $f$ (and thus $f^{-1}$) is invariant across environments. Therefore, the Jacobian determinant term $\log |\det J_{f^{-1}}(\mathbf{x})|$ and the Hessian of the inverse map $D_{\mathbf{x}}^2 f_k^{-1}(\mathbf{x})$ are identical for both environments.
    
    Subtracting the expanded Hessians for environment $i$ from the base environment:
    \begin{align}
        D_{\mathbf{x}}^2 \log p(\mathbf{x}) - D_{\mathbf{x}}^2 \log p^{(i)}(\mathbf{x}) &= J_{f^{-1}}(\mathbf{x})^\top \left( D_{\mathbf{s}}^2 \log p_{\mathbf{S}}(\mathbf{s}) - D_{\mathbf{s}}^2 \log p_{\mathbf{S}^{(i)}}(\mathbf{s}) \right) J_{f^{-1}}(\mathbf{x}) \nonumber \\
        &\quad + \sum_{k=1}^d \left( \frac{\partial \log p_{\mathbf{S}}}{\partial s_k}(\mathbf{s}) - \frac{\partial \log p_{\mathbf{S}^{(i)}}}{\partial s_k}(\mathbf{s}) \right) D_{\mathbf{x}}^2 f_k^{-1}(\mathbf{x})
    \end{align}
    Note that the term $D_{\mathbf{x}}^2 \log |\det J_{f^{-1}}(\mathbf{x})|$ has cancelled out.
    
    We proceed by evaluating this difference at the specific reference point $\mathbf{x}^*$ such that $\mathbf{s}^* = f^{-1}(\mathbf{x}^*) = \mu_{\mathbf{S}}$.
    For centered distributions like the Gaussian or Student's t (with $\nu > 1$), the gradient of the log-density vanishes at the mode $\mu_{\mathbf{S}}$:
    \begin{equation}
        \nabla_{\mathbf{s}} \log p_{\mathbf{S}}(\mu_{\mathbf{S}}) = \mathbf{0} \quad \text{and} \quad \nabla_{\mathbf{s}} \log p_{\mathbf{S}^{(i)}}(\mu_{\mathbf{S}}) = \mathbf{0}
    \end{equation}
    Consequently, the summation term involving the first derivatives vanishes at $\mathbf{x}^*$. We are left with only the first term:
    \begin{equation}
        D_{\mathbf{x}}^2 \log p(\mathbf{x}^*) - D_{\mathbf{x}}^2 \log p^{(i)}(\mathbf{x}^*) = J_{f^{-1}}(\mathbf{x}^*)^\top \left( D_{\mathbf{s}}^2 \log p_{\mathbf{S}}(\mathbf{s}^*) - D_{\mathbf{s}}^2 \log p_{\mathbf{S}^{(i)}}(\mathbf{s}^*) \right) J_{f^{-1}}(\mathbf{x}^*)
    \end{equation}
    By defining the diagonal matrix $\mathbf{\Omega}^{(s,i)}$ as the difference of the source Hessians:
    \begin{equation}
        \mathbf{\Omega}^{(s,i)} := D_{\mathbf{s}}^2 \log p_{\mathbf{S}}(\mu_{\mathbf{S}}) - D_{\mathbf{s}}^2 \log p_{\mathbf{S}^{(i)}}(\mu_{\mathbf{S}})
    \end{equation}
    we obtain the final result stated in Equation (15):
    \begin{equation}
        D_{\mathbf{x}}^2 \log p(\mathbf{x}^*) - D_{\mathbf{x}}^2 \log p^{(i)}(\mathbf{x}^*) = J_{f^{-1}}(\mathbf{x}^*)^\top \mathbf{\Omega}^{(s,i)} J_{f^{-1}}(\mathbf{x}^*)
    \end{equation}
\end{proof}
\section{Proof of temporal-environmental identifiability bound}
\label{appendix_temp_iden}
\begin{theorem}
    Consider a time series Structural Causal Model $X_t = f(S_t)$ with $d$ variables, where the mixing function $f$ is time-invariant within a computational window of size $T$. Let there be $k$ auxiliary environments exhibiting \textbf{insufficient heterogeneity} (Definition \ref{suffinsuffhetero}), such that the environmental variance shifts at any single time step identify a subspace of rank at most $r < d$.
    
    Suppose the Jacobian $J_{f^{-1}}$ satisfies \textbf{Temporal Faithfulness} (Definition \ref{tempfaith}) and the variance profiles satisfy the \textbf{Distinct Variance Profiles} (Definition \ref{distvarianceprofs}) assumption. Then, no method relying on second order moments can identify the causal graph unless:
    \begin{equation}
        T \ge \left\lceil \frac{d}{r} \right\rceil
    \end{equation}
    Moreover, this bound is tight under generic temporal variance diversity, i.e., when each timestep contributes a rank-x
     variance subspace in general position so that the accumulated span grows maximally with time.
\end{theorem}
 
\begin{proof}
    The proof proceeds by constructing a sufficient statistic, inspired from the work of \cite{montagna2025identifiabilitycausalgraphsmultiple}, that aggregates second-order information across the temporal window to recover the full rank of the mixing mechanism. Since the mixing function $f$ is time-invariant within the window $[t-T, t]$, the Jacobian $J_{f^{-1}}(x)$ evaluated at a fixed reference point $x^*$ remains constant. This allows us to accumulate structural constraints over time.
    
    We begin by examining the Hessian of the log-density. Under the assumption of an ICA model with independent sources, the Hessian of the observational log-density $\log p_X(x)$ relates to the source density via the chain rule. Specifically, evaluating the difference between the Hessians of the base environment and an auxiliary environment $i$ at a fixed point $x^*$ yields the relation by lemma \ref{lemma1}:
    \begin{equation}
        D_x^2 \log p(x^*) - D_x^2 \log p^{(i)}(x^*) = J_{f^{-1}}(x^*)^\top \Omega^{(s,i)} J_{f^{-1}}(x^*)
    \end{equation}
    where $\Omega^{(s,i)} \in \mathbb{R}^{d \times d}$ is a diagonal matrix representing the shift in the precision (inverse variance) of the latent sources at time $s$ in environment $i$.
    
    To leverage temporal information, we define the \textit{Temporal-Environmental Hessian} $\Psi_{TE}$ by aggregating these differences over the computational window $T$ and across all $k$ environments:
    \begin{equation}
        \Psi_{TE} = \sum_{s=t-T}^{t} \sum_{i=1}^{k} \left( D_x^2 \log p(x^*) - D_x^2 \log p^{(i)}(x^*) \right)
    \end{equation}
    Substituting the Jacobian relation into the sum and factoring out the common matrix $J_{f^{-1}}(x^*)$, we obtain:
    \begin{equation}
        \Psi_{TE} = J_{f^{-1}}(x^*)^\top \left( \sum_{s=t-T}^{t} \sum_{i=1}^{k} \Omega^{(s,i)} \right) J_{f^{-1}}(x^*)
    \end{equation}
    
    Let $\Lambda_{\text{total}} = \sum_{s,i} \Omega^{(s,i)}$ denote the central diagonal matrix. The invertibility of $\Psi_{TE}$ depends entirely on whether $\Lambda_{\text{total}}$ attains full rank $d$. By the definition of \textit{insufficient heterogeneity}, the information contribution from any single time step $s$ is rank-deficient, bounded by $\text{rank}(\sum_{i} \Omega^{(s,i)}) \le r < d$.
    
    We thereby derive a fundamental lower bound on the computational window size. Utilizing the sub-additivity of rank, the rank of the aggregated matrix is bounded by the sum of the individual ranks. Even in the \textit{optimal scenario} where the variance profiles at distinct time steps lie in strictly disjoint subspaces and thus provide maximal information gain, the accumulated rank cannot exceed the linear sum of the inputs:
    \begin{equation}
        \text{rank}(\Lambda_{\text{total}}) \le \sum_{s=1}^T \text{rank}\left(\sum_{i=1}^k \Omega^{(s,i)}\right) \le T \cdot r
    \end{equation}
    This inequality implies that regardless of the temporal diversity quality, a window size of $T < d/r$ fundamentally precludes identifiability. Consequently, satisfying $T \ge \lceil d/r \rceil$ is a necessary (but not sufficient) condition to span the full $d$-dimensional space. Only in the most optimal case where each step contributes maximal novel information, this window size is sufficient to ensure $\text{rank}(\Lambda_{\text{total}}) = d$, making $\Psi_{TE}$ invertible.
    
    With invertibility established, we employ joint diagonalization. Partitioning the aggregated information into two groups yields two invertible matrices $\Psi_{TE,1}$ and $\Psi_{TE,2}$. The product $M = \Psi_{TE,1}^{-1} \Psi_{TE,2}$ shares the same eigenvectors as the system:
    \begin{equation}
        M = J^{-1} (\Lambda_1^{-1} \Lambda_2) J
    \end{equation}
    By the assumption of \textit{Distinct Variance Profiles}, the eigenvalues (diagonal entries of $\Lambda_1^{-1} \Lambda_2$) are distinct, ensuring that the eigendecomposition of $M$ uniquely recovers the columns of $J_{f}(x^*)$ up to permutation and scaling. Finally, the \textit{Temporal Faithfulness} of the Jacobian ensures that the support of the recovered matrix corresponds uniquely to the edges of the causal graph $\mathcal{G}$.
\end{proof}
\section{Extension of identifiability proof to Multivariate Student's t Noise}
\textbf{Heavy-Tailed Sampling (Student's t).} For the heavy-tailed regime, we 
model the noise $\epsilon_t$ using a Multivariate Student's t distribution with 
degrees of freedom $\nu$ and \textbf{diagonal} scale matrix 
$\Sigma = \text{diag}(\sigma_{1,t}^{(e)2}, \ldots, \sigma_{d,t}^{(e)2})$.

\textbf{Remark on ICA Compatibility:} While this distribution does not have 
fully independent components (due to shared tail dependence), it satisfies 
the requirements for our identifiability theory. Specifically, our proof 
(Theorem 4.2) only requires that the Hessian of the log-density is diagonal 
when evaluated at the mode $\mu_S = 0$ (Temporal Faithfulness, Definition 3.7). 
The multivariate Student's t with diagonal $\Sigma$ satisfies this property:
\begin{equation}
D^2_s \log p_S(s)\big|_{s=0} = -\frac{\nu+d}{\nu}\Sigma^{-1} \quad \text{(diagonal)}
\end{equation}
while also providing the heavy-tailed behavior necessary for our finite-sample 
analysis.
\label{app:student_t_identifiability}
 
\textbf{Theorem.} \textit{The necessary identifiability condition $T \ge \lceil d/r \rceil$ derived in Theorem \ref{thm:temporal_identifiability} holds invariant for latent sources following a Multivariate Student's t distribution.}
 
\begin{proof}
Recall from Lemma 4.1 that structural identifiability relies on the rank of the aggregated Hessian difference matrix $\Psi_{TE} = J^{\top} \left( \sum_{s,i} \Omega^{(s,i)} \right) J$. A sufficient condition for unique recovery is that the central matrix of precision shifts $\sum \Omega^{(s,i)}$ attains full rank $d$. We show here that the geometric structure of $\Omega^{(s,i)}$ under Student's t noise is identical to the Gaussian case up to a non-zero scalar factor.
 
Let the noise vector $\epsilon_t$ follow a zero-mean multivariate Student's t-distribution $t_{\nu}(0, \Sigma)$ as defined in Eq. (4). The log-likelihood is given by:
\begin{equation}
    \log p(\epsilon) = C - \frac{\nu+d}{2} \log \left( 1 + \frac{1}{\nu} \epsilon^{\top} \Sigma^{-1} \epsilon \right)
\end{equation}
We evaluate the Hessian $H(\epsilon) = \nabla^2 \log p(\epsilon)$ at the mode $\epsilon^* = 0$ to satisfy the vanishing gradient condition of Lemma A.1. First, the score function is:
\begin{equation}
    \nabla \log p(\epsilon) = - \frac{\nu+d}{2} \cdot \frac{1}{1 + \frac{1}{\nu}\epsilon^{\top}\Sigma^{-1}\epsilon} \cdot \frac{2}{\nu} \Sigma^{-1} \epsilon
\end{equation}
At the mode $\epsilon = 0$, the gradient vanishes ($\nabla \log p(0) = 0$). Differentiating again to find the Hessian at $\epsilon=0$:
\begin{equation}
    H(0) = - \frac{\nu+d}{\nu} \Sigma^{-1}
\end{equation}
Now, consider the Hessian difference matrix $\Omega^{(s,i)}$ between the base environment ($u=0$) and auxiliary environment ($u=i$) at time $s$. Let $\Sigma_{(0)}$ and $\Sigma_{(i)}$ be the diagonal scale matrices of the latent sources in these environments.
\begin{align}
    \Omega^{(s,i)} &= H_{(0)}(0) - H_{(i)}(0) \\
    &= \left( - \frac{\nu+d}{\nu} \Sigma_{(0)}^{-1} \right) - \left( - \frac{\nu+d}{\nu} \Sigma_{(i)}^{-1} \right) \\
    &= \frac{\nu+d}{\nu} \left( \Sigma_{(i)}^{-1} - \Sigma_{(0)}^{-1} \right)
\end{align}
Let $\Lambda^{(s,i)} = \Sigma_{(i)}^{-1} - \Sigma_{(0)}^{-1}$ be the matrix of precision shifts. The structural information provided by this environment is:
\begin{equation}
    \Omega^{(s,i)} \propto \Lambda^{(s,i)}
\end{equation}
Since $\frac{\nu+d}{\nu}$ is a strictly non-zero scalar constant for $\nu > 0$, the rank of $\Omega^{(s,i)}$ is determined entirely by the rank of the precision shift matrix $\Lambda^{(s,i)}$. This is the exact same rank condition as in the Gaussian case (where the scalar is 1). Consequently, the accumulation of rank over time follows, and the identifiability threshold $T \ge \lceil d/r \rceil$ remains necessary.
\end{proof}
\section{Fourth-Order Mixed Moments of the Multivariate Student-$t$ Distribution}
\label{appendix2}
\noindent \textbf{Proposition.} \textit{Let $X \sim t_\nu(0, \Sigma)$ be a multivariate Student-$t$ distributed random vector with degrees of freedom $\nu > 4$ and covariance matrix $\Sigma$. Assuming uncorrelated components (or $j=k$), the fourth-order mixed moments satisfy:}
\begin{equation}
    E[X_j^2 X_k^2] = \frac{\nu-2}{\nu-4} (1 + 2\delta_{jk}) \Sigma_{jj}\Sigma_{kk}
\end{equation}
 
\begin{proof}
\noindent \textbf{1. Stochastic Representation} \\
The vector $X$ admits the representation as a Gaussian scale mixture:
\begin{equation}
    X \overset{d}{=} \sqrt{Y} Z
\end{equation}
where:
\begin{itemize}
    \item $Z \sim \mathcal{N}(0, S)$ is a multivariate Gaussian with scale matrix $S$.
    \item $Y = \nu W^{-1}$ where $W \sim \chi^2_\nu$ is independent of $Z$.
\end{itemize}
 
\noindent \textbf{2. Covariance Normalization} \\
We first relate the scale matrix $S$ to the true covariance $\Sigma$.
\begin{equation}
    \Sigma = Cov(X) = E[X X^\top] = E[Y] E[Z Z^\top] = E[Y] S
\end{equation}
For $W \sim \chi^2_\nu$, the first inverse moment is $E[W^{-1}] = (\nu-2)^{-1}$. Thus:
\begin{equation}
E[Y] = \nu E[W^{-1}] = \frac{\nu}{\nu-2} \quad \implies \quad S = \frac{\nu-2}{\nu} \Sigma
\end{equation}
 
\noindent \textbf{3. Moments of the Mixing Variable} \\
For the fourth moment of $X$, we require the second moment of $Y$:
\begin{equation}
    E[Y^2] = \nu^2 E[W^{-2}] = \frac{\nu^2}{(\nu-2)(\nu-4)}
\end{equation}
 
\noindent \textbf{4. Gaussian Fourth Moments} \\
By Isserlis' Theorem (Wick's formula) for zero-mean Gaussian vector $Z$:
\begin{equation}
    E[Z_j^2 Z_k^2] = S_{jj}S_{kk} + 2S_{jk}^2
\end{equation}
 
\noindent \textbf{5. Derivation of Student-$t$ Moments} \\
Exploiting the independence of $Y$ and $Z$:
\begin{align}
    E[X_j^2 X_k^2] &= E\left[(\sqrt{Y} Z_j)^2 (\sqrt{Y} Z_k)^2\right] \nonumber \\
    &= E[Y^2] \cdot E[Z_j^2 Z_k^2] \nonumber \\
    &= \frac{\nu^2}{(\nu-2)(\nu-4)} \left( S_{jj}S_{kk} + 2S_{jk}^2 \right)
\end{align}
Substituting $S = \frac{\nu-2}{\nu} \Sigma$ from Step 2:
\begin{align}
    E[X_j^2 X_k^2] &= \frac{\nu^2}{(\nu-2)(\nu-4)} \left[ \left(\frac{\nu-2}{\nu}\Sigma_{jj}\right)\left(\frac{\nu-2}{\nu}\Sigma_{kk}\right) + 2\left(\frac{\nu-2}{\nu}\Sigma_{jk}\right)^2 \right] \nonumber \\
    &= \frac{\nu^2}{(\nu-2)(\nu-4)} \cdot \left(\frac{\nu-2}{\nu}\right)^2 \left[ \Sigma_{jj}\Sigma_{kk} + 2\Sigma_{jk}^2 \right] \nonumber \\
    &= \frac{\nu-2}{\nu-4} \left[ \Sigma_{jj}\Sigma_{kk} + 2\Sigma_{jk}^2 \right]
\end{align}
Under the assumption that $\Sigma_{jk} = 0$ for $j \neq k$ (or considering the diagonal case $j=k$), we utilize the Kronecker delta $\delta_{jk}$ to obtain the final form:
\begin{equation}
    E[X_j^2 X_k^2] = \frac{\nu-2}{\nu-4} (1 + 2\delta_{jk}) \Sigma_{jj}\Sigma_{kk}
\end{equation}
\end{proof}
\section{Frobenius Norm Error Bound via Fourth-Moment Contraction}
\label{appendix3}
\noindent \textbf{Proposition.} \textit{Let $\hatSigma_N$ be the sample covariance estimator constructed from $N$ independent samples of a zero-mean random vector $X \in \mathbb{R}^d$ with covariance $\bSigma$. Assuming the distribution follows an elliptical model with kurtosis parameter $\kappa = \frac{\nu-2}{\nu-4}$ (where $\kappa=1$ for Gaussian), the expected squared Frobenius norm error is given by:}
\begin{equation}
    \E\left[ \|\hatSigma_N - \bSigma\|_F^2 \right] = \frac{1}{N} \left[ \kappa (\Tr(\bSigma))^2 + (2\kappa - 1) \|\bSigma\|_F^2 \right]
\end{equation}
\textit{This demonstrates that the convergence rate is governed by the sample size $N$ and scaled by the kurtosis parameter $\kappa$ and the spectral profile of $\bSigma$.}
 
\begin{proof}
\noindent \textbf{1. Decomposition of the Frobenius Norm Error} \\
The squared Frobenius norm of the error matrix is the sum of the squared errors of its individual entries. By linearity of expectation:
\begin{equation}
    \E\left[ \|\hatSigma_N - \bSigma\|_F^2 \right] = \E\left[ \sum_{j,k} (\hat{\sigma}_{jk} - \sigma_{jk})^2 \right] = \sum_{j,k} \Var(\hat{\sigma}_{jk})
\end{equation}
where $\hat{\sigma}_{jk} = \frac{1}{N} \sum_{i=1}^N X_{ij}X_{ik}$ is the sample covariance element for dimensions $j$ and $k$.
 
\noindent \textbf{2. Variance of a Single Entry} \\
Since samples are i.i.d., the variance of the mean estimator scales with $1/N$:
\begin{equation}
    Var(\hat{\sigma}_{jk}) = \frac{1}{N} \Var(X_j X_k) = \frac{1}{N} \left( \E[X_j^2 X_k^2] - (\E[X_j X_k])^2 \right)
\end{equation}
We know that $\E[X_j X_k] = \sigma_{jk}$. The fourth-order moment $\E[X_j^2 X_k^2]$ depends on the kurtosis of the distribution.
 
\noindent \textbf{3. Fourth-Moment Tensor Contraction} \\
Using the general fourth-moment relation derived for elliptical distributions (Appendix B), the mixed moment is:
\begin{equation}
    \E[X_j^2 X_k^2] = \kappa \left( \sigma_{jj}\sigma_{kk} + 2\sigma_{jk}^2 \right)
\end{equation}
Substituting this back into the variance expression:
\begin{align}
    \Var(\hat{\sigma}_{jk}) &= \frac{1}{N} \left[ \kappa (\sigma_{jj}\sigma_{kk} + 2\sigma_{jk}^2) - \sigma_{jk}^2 \right] \nonumber \\
    &= \frac{1}{N} \left[ \kappa \sigma_{jj}\sigma_{kk} + (2\kappa - 1) \sigma_{jk}^2 \right]
\end{align}
 
\noindent \textbf{4. Summation Over the Full Matrix} \\
To obtain the full Frobenius error, we sum the variance expression over all indices $j, k \in \{1, \dots, d\}$:
\begin{equation}
    \E\left[ \|\hatSigma_N - \bSigma\|_F^2 \right] = \frac{1}{N} \left[ \sum_{j,k} \kappa \sigma_{jj}\sigma_{kk} + \sum_{j,k} (2\kappa - 1) \sigma_{jk}^2 \right]
\end{equation}
 
\noindent \textbf{5. Identification of Matrix Norms} \\
We identify the tensor contractions with standard matrix invariants:
\begin{itemize}
    \item The first term decouples: $\sum_{j,k} \sigma_{jj}\sigma_{kk} = (\sum_j \sigma_{jj})(\sum_k \sigma_{kk}) = (\Tr(\bSigma))^2$.
    \item The second term is the definition of the Frobenius norm: $\sum_{j,k} \sigma_{jk}^2 = \|\bSigma\|_F^2$.
\end{itemize}
Substituting these into the summation:
\begin{equation}
    \E\left[ \|\hatSigma_N - \bSigma\|_F^2 \right] = \frac{1}{N} \left[ \kappa (\Tr(\bSigma))^2 + (2\kappa - 1) \|\bSigma\|_F^2 \right]
\end{equation}
 
\noindent \textbf{6. Analysis of Dominant Scaling} \\
For heavy-tailed distributions where $\nu \to 4^+$, the parameter $\kappa \to \infty$. The error is thus dominated by the kurtosis term:
\begin{equation}
    \lim_{\kappa \to \infty} \E\left[ \|\hatSigma_N - \bSigma\|_F^2 \right] \approx \frac{\kappa}{N} \left( (\Tr(\bSigma))^2 + 2\|\bSigma\|_F^2 \right)
\end{equation}
This confirms that the convergence rate is dictated by the kurtosis inflation factor $\kappa$, requiring $N \gg \kappa$ to ensure concentration.
\end{proof}
 
\section{Local Asymptotic Normality of the Sample Covariance}
\label{app:lan}

\begin{lemma}[LAN for the Sample Covariance]
\label{lem:statistic_LAN}
Let $\{X_i\}_{i=1}^N$ be i.i.d.\ random vectors in $\mathbb{R}^d$ with distribution
$P_\Sigma$, mean zero, and covariance $\Sigma$. Assume the family
$\{P_\Sigma : \Sigma \in \mathcal{S}_{++}^d\}$ is Differentiable in Quadratic Mean
(DQM) and regular in the sense of Le Cam.

Define the sample covariance
\[
\hat{\Sigma}_N = \frac{1}{N} \sum_{i=1}^N X_i X_i^\top .
\]
Then the sequence of experiments generated by observing $\hat{\Sigma}_N$ is
locally asymptotically normal. In particular,
\[
\sqrt{N}\,\mathrm{vec}(\hat{\Sigma}_N - \Sigma)
\;\Rightarrow\;
\mathcal{N}\!\left(0, \Gamma_\nu(\Sigma)\right),
\]
where $\Gamma_\nu(\Sigma) = \mathrm{Cov}(\mathrm{vec}(X X^\top))$.

Moreover, for local alternatives $\Sigma_k = \Sigma_j + \Delta/\sqrt{N}$,
the Kullback--Leibler divergence satisfies
\begin{equation}
\label{eq:lan_kl}
D_{\mathrm{KL}}\!\left(
P_{\Sigma_j}^N \,\|\, P_{\Sigma_k}^N
\right)
=
\frac{N}{2}
\langle
\Delta, [\Gamma_\nu(\Sigma_j)]^{-1} \Delta
\rangle
+ o(1).
\end{equation}
\end{lemma}

\begin{proof}
Let $Z_i = \mathrm{vec}(X_i X_i^\top)$ and $\mu(\Sigma) = \mathbb{E}_\Sigma[Z_i] =
\mathrm{vec}(\Sigma)$. Then
\[
\mathrm{vec}(\hat{\Sigma}_N)
=
\frac{1}{N} \sum_{i=1}^N Z_i .
\]
Since $\mathbb{E}_\Sigma[\|Z_i\|^2] < \infty$ for $\nu>4$, the multivariate CLT
yields
\[
\sqrt{N}\left(\mathrm{vec}(\hat{\Sigma}_N) - \mu(\Sigma)\right)
\Rightarrow
\mathcal{N}(0,\Gamma_\nu(\Sigma)).
\]

By the DQM and regularity assumptions, the induced sequence of experiments
observing $\hat{\Sigma}_N$ converges, in the sense of Le Cam, to the Gaussian
shift experiment
\[
Y = h + W, \qquad W \sim \mathcal{N}(0,\Gamma_\nu(\Sigma)),
\]
under local perturbations $\Sigma \mapsto \Sigma + h/\sqrt{N}$.

The Kullback--Leibler divergence between two Gaussian shift experiments with
means $0$ and $h$ and common covariance $\Gamma_\nu(\Sigma)$ is
\[
D_{\mathrm{KL}}
=
\frac{1}{2} h^\top [\Gamma_\nu(\Sigma)]^{-1} h.
\]
Setting $h = \Delta$ and rescaling by $N$ yields \eqref{eq:lan_kl}.
\end{proof}
\newpage
\section{Evaluation Metrics}
\label{app:metrics}

To rigorously assess the performance of our causal discovery framework, we utilize two complementary metrics: one for topological accuracy (SHD) and one for the reliability of confidence rankings (AUDRC).

\subsection{Structural Hamming Distance (SHD)}
The Structural Hamming Distance (SHD) quantifies the topological disagreement between the estimated graph $\hat{G}$ and the ground truth graph $G$. It is defined as the minimum number of edge insertions, deletions, or flips required to transform $\hat{G}$ into $G$.

Let $B$ and $\hat{B}$ be the binary adjacency matrices of the true and estimated graphs, respectively, where $B_{ij}=1$ denotes an edge $j \to i$. The SHD is calculated as:
\begin{equation}
    \text{SHD}(B, \hat{B}) = \sum_{i \neq j} \mathbb{1}\left( B_{ij} \neq \hat{B}_{ij} \right)
\end{equation}
A lower SHD indicates better performance ($\downarrow$), with 0 representing perfect recovery. For our experiments, we threshold the continuous edge weights $W$ at a fixed cutoff $\tau$ (e.g., $0.3$) to obtain the binary matrix $\hat{B}$.

\subsection{Area Under the Decision Rate Curve (AUDRC)}
While SHD evaluates a single static graph snapshot, it relies on an arbitrary threshold $\tau$. To evaluate the quality of the learned edge weights across \textit{all possible sparsity levels}, we employ the Area Under the Decision Rate Curve (AUDRC).

The AUDRC measures how well the estimator's confidence scores rank true edges against non-existent ones. Let $M = d(d-1)$ be the total number of possible directed edges in a $d$-dimensional graph. Let $\pi$ be a permutation of the $M$ possible edges sorted in descending order of their absolute estimated weights $|\hat{W}_{ij}|$. The AUDRC is defined as the average precision computed iteratively as we include the top-$m$ most confident edges:

\begin{equation}
    \text{AUDRC} = \frac{1}{M} \sum_{m=1}^{M} \left( \frac{1}{m} \sum_{k=1}^{m} \mathbb{1}\left( \text{is\_true}(\pi(k)) \right) \right)
\end{equation}

where $\mathbb{1}(\text{is\_true}(\pi(k)))$ is an indicator function that equals 1 if the $k$-th most confident edge exists in the ground truth graph, and 0 otherwise.

\paragraph{Interpretation:}
\begin{itemize}
    \item \textbf{Threshold Independence:} Unlike F1 or SHD, AUDRC does not require selecting a hyperparameter $\tau$. It evaluates the precision at every single sparsity level $m \in \{1, \dots, M\}$.
    \item \textbf{Ranking Quality:} A high AUDRC ($\uparrow$, max 1.0) implies that the model consistently assigns higher weights to true causal links than to spurious ones. This is critical for real-world applications where practitioners prioritize the most confident causal predictions.
\end{itemize}
\newpage
\section{Algorithmic Implementation}
\label{algo}
\begin{algorithm}
\caption{Identifiability under multi-environmental time-series data}
   \label{alg:te_jd}
\begin{algorithmic}[1]
   \INPUT Data $\mathbf{X} = \{X_t\}_{t=1}^{T_{total}}$ partitioned into $k$ environments; Window Size $T \ge \lceil d/x \rceil$; Degrees of Freedom $\nu$; Confidence $\delta, \epsilon$.
   \OUTPUT Estimated Causal Adjacency Matrix $\hat{B}$.
 
   \STATE \textbf{Phase 1: Finite-Sample Stability Check}
   \STATE $N_{Gauss} \leftarrow \mathcal{O}(\frac{1}{\epsilon^2} \log(1/\delta))$
   \IF{$\nu > 4$}
       \STATE $\gamma \leftarrow \TailPenalty{\nu}$ \COMMENT{Calculate Heavy-Tailed Penalty Factor}
       \STATE $N_{req} \leftarrow N_{Gauss} \times \gamma$ \COMMENT{Scale by kurtosis inflation}
   \ELSE
       \STATE $N_{req} \leftarrow N_{Gauss}$ \COMMENT{Standard Gaussian regime}
   \ENDIF
   
   \STATE $N_{bin} \leftarrow$ Samples per temporal bin in $\mathbf{X}$
   \IF{$N_{bin} < N_{req}$}
       \STATE \textbf{Error:} Insufficient samples to bound error $\|\hat{\Sigma} - \Sigma\|_F \le \epsilon$. Increase bin size.
   \ENDIF
 
   \STATE
   \STATE \textbf{Phase 2: Estimation of Structural Shifts}
   \STATE Initialize Aggregate Hessian $\Psi_{TE} \leftarrow \mathbf{0}_{d \times d}$
   \STATE Compute Global Precision $\hat{\Theta}_{global} \leftarrow (\frac{1}{T_{total}}\sum X_t X_t^\top)^{-1}$
   
   \FOR{time step $s = 1$ to $T$}
       \FOR{environment $i = 1$ to $k$}
           \STATE $\hat{\Sigma}^{(s,i)} \leftarrow \text{Cov}(X^{(s,i)})$ \COMMENT{Estimate Local Covariance}
           \STATE $\hat{\Theta}^{(s,i)} \leftarrow (\hat{\Sigma}^{(s,i)})^{-1}$ \COMMENT{Estimate Local Precision}
           \STATE $\Omega^{(s,i)} \leftarrow \hat{\Theta}^{(s,i)} - \hat{\Theta}_{global}$ \COMMENT{Compute Shift Matrix}
           \STATE $\Psi_{TE} \leftarrow \Psi_{TE} + \Omega^{(s,i)}$ \COMMENT{Accumulate Information}
       \ENDFOR
   \ENDFOR
 
   \STATE
   \STATE \textbf{Phase 3: Joint Diagonalization \& Recovery (Thm 4.1)}
   \STATE $\hat{A} \leftarrow \text{JointDiag}(\{\Omega^{(s,i)}\}_{s,i})$ \COMMENT{Recover Mixing Matrix $A$}
   
   \STATE $\hat{W} \leftarrow \hat{A}^{-1}$ \COMMENT{Compute Unmixing Matrix}
   \STATE $\hat{W}_{perm} \leftarrow \text{PermuteAndScale}(\hat{W})$ \COMMENT{Enforce lower-triangularity}
   \STATE $\hat{B} \leftarrow I - \hat{W}_{perm}$
 
   \STATE \textbf{return} $\hat{B}$
\end{algorithmic}
\end{algorithm}
\end{document}